\documentclass{article}
\pdfoutput=1 

\usepackage[preprint, nonatbib]{neurips_2026}

\usepackage[numbers,sort&compress]{natbib}

\usepackage[utf8]{inputenc} \usepackage[T1]{fontenc}    \usepackage{url}            \usepackage{booktabs}       \usepackage{amsfonts}       \usepackage{nicefrac}       \usepackage{microtype}      \usepackage{xcolor}         

\colorlet{siaminlinkcolor}{green!50!black}
\colorlet{siamexlinkcolor}{red!50!black}
\colorlet{siamreviewcolor}{black!50}
\usepackage[colorlinks,allcolors=siaminlinkcolor,urlcolor=siamexlinkcolor,pagebackref,bookmarksdepth=3]{hyperref}
\renewcommand*{\backref}[1]{\ifx#1\relax \else Page #1 \fi}
\renewcommand*{\backrefalt}[4]{\ifcase #1 \footnotesize{(Not cited)}\or        \footnotesize{(Cited on page~#2)}\else      \footnotesize{(Cited on pages~#2)}\fi
}

\usepackage{etoolbox}

\makeatletter
\patchcmd{\NAT@citexnum}
  {\let\NAT@last@num\NAT@num}
  {\@ifundefined{MakeLinkTarget}{\phantomsection}{\MakeLinkTarget[cite]{}}\Hy@backout{\@citeb\@extra@b@citeb}\let\NAT@last@num\NAT@num
  }
  {\typeout{natbib/hyperref backref patch applied.}}
  {\PackageError{natbib-hyperref-backref-patch}
      {Could not patch \string\NAT@citexnum}
      {The natbib/hyperref internals differ from the expected version.}}
\makeatother

\usepackage{amsmath}
\usepackage{amsthm}
\usepackage{thmtools}
\usepackage{ifpdf}
\usepackage{graphicx}
\usepackage{algorithmic}
\usepackage{amsopn}

\usepackage{amssymb}
\usepackage{subcaption}
\usepackage{mathtools}
\usepackage{bbm}
\usepackage{mathrsfs}
\usepackage{stmaryrd}
\usepackage{thm-restate}
\usepackage{makecell}

\let\originalleft\left
\let\originalright\right
\renewcommand{\left}{\mathopen{}\mathclose\bgroup\originalleft}
\renewcommand{\right}{\aftergroup\egroup\originalright}
\usepackage{array}
\newcolumntype{M}[1]{>{\centering\arraybackslash}m{#1}}
\usepackage{tcolorbox}
\usepackage{bm}

\usepackage[capitalize]{cleveref}
\crefname{assumption}{Assumption}{Assumptions}

\newcommand{\N}{\mathbb{N}}

\newcommand{\R}{\mathbb{R}}

\newcommand{\cL}{\mathcal{L}}
\newcommand{\cF}{\mathcal{F}}

\newcommand{\cH}{\mathcal{H}}
\newcommand{\cX}{\mathcal{X}}

\newcommand{\cC}{\mathcal{C}}
\newcommand{\cS}{\mathcal{S}}
\newcommand{\cG}{\mathcal{G}}

\newcommand{\cR}{\mathcal{R}}

\newcommand{\cU}{\mathcal{U}}

\newcommand{\cZ}{\mathcal{Z}}
\newcommand{\cQ}{\mathcal{Q}}
\newcommand{\cN}{\mathcal{N}}
\newcommand{\cM}{\mathcal{M}}

\newcommand{\sP}{\mathscr{P}}

\newcommand{\sfd}{\mathsf{d}}

\newcommand{\Cac}{C_{\mathrm{ac}}}
\newcommand{\Lac}{L_{\mathrm{ac}}}

\DeclarePairedDelimiterX{\iptemp}[2]{\langle}{\rangle}{#1, #2}

\DeclarePairedDelimiterX{\normtemp}[1]{\lVert}{\rVert}{#1}
\newcommand{\norm}{\normtemp}
\DeclarePairedDelimiterX{\abstemp}[1]{\lvert}{\rvert}{#1}
\newcommand{\abs}{\abstemp}
\DeclarePairedDelimiterX{\trtemp}[1]{(}{)}{#1}

\DeclarePairedDelimiterX{\SEtemp}[2]{(}{)}{#1, #2}

\DeclarePairedDelimiterX{\SGtemp}[1]{(}{)}{#1}

\newcommand{\defeq}{\coloneqq}
\newcommand{\eqdef}{\eqqcolon}

\newcommand{\condbar}{\, \vert \,}

\DeclarePairedDelimiterX{\floor}[1]{\lfloor}{\rfloor}{#1}
\DeclarePairedDelimiterX{\ceil}[1]{\lceil}{\rceil}{#1}

\DeclareMathOperator{\domain}{Dom}

\newcommand{\E}{\operatorname{\mathbb{E}}}

\newcommand{\one}{\mathbbm{1}}

\newcommand{\normal}{\mathcal{N}}

\newcommand{\iunit}{\mathrm{i}}
\newcommand{\diff}{d}

\newcommand{\dd}[1]{\,\diff{#1}}

\DeclareMathOperator*{\essinf}{ess\,inf}

\newcommand{\weakly}{\rightsquigarrow}

\newcommand{\al}{\alpha}
\newcommand{\ep}{\varepsilon}
\newcommand{\Rd}{\mathbb{R}^d}

\newcommand{\eps}{\epsilon}

\renewcommand{\epsilon}{\ep}

\def\qfa{\quad\text{for all}\quad}

\def\qas{\quad\text{as}\quad}
\def\qa{\quad\text{and}\quad}

\def\qw{\quad\text{where}\quad}

\newcommand{\set}[2]{{\left\{ #1 \,\middle|\, #2 \right\}}}
\newcommand{\slot}{{\,\cdot\,}}

\ifpdf
\DeclareGraphicsExtensions{.eps,.pdf,.png,.jpg,.jpeg}
\else
\DeclareGraphicsExtensions{.eps}
\fi
\graphicspath{ {./figures/} }

\usepackage[shortlabels]{enumitem}

 \setlist[itemize]{leftmargin=15pt}
\setlist[enumerate]{leftmargin=15pt}
\renewcommand{\tilde}{\widetilde}

\newtheorem{theorem}{Theorem}[section]
\newtheorem{corollary}[theorem]{Corollary}

\newtheorem{lemma}[theorem]{Lemma}

\theoremstyle{definition}
\newtheorem{definition}[theorem]{Definition}
\newtheorem{remark}[theorem]{Remark}
\newtheorem{example}[theorem]{Example}



\newcommand{\mytitle}{
    One Operator for Many Densities: Amortized \\
    Approximation of Conditioning by Neural Operators
}

\title{\mytitle}

\makeatletter
\let\inserttitle\@title
\makeatother

\author{
  Panos Tsimpos \\ 
  Operations Research Center \\
  Massachusetts Institute of Technology \\
  Cambridge, MA 02139 \\
  \texttt{ptsimpos@mit.edu}
  \And
  Edoardo Calvello\\
  Department of Computing and Mathematical Sciences \\
  California Institute of Technology \\
  Pasadena, CA 91125 \\
  \texttt{e.calvello@caltech.edu}
  \And
  Ayoub Belhadji\\
  Laboratory for Information and Decision Systems \\
  Center for Computational Science and Engineering\\
  Massachusetts Institute of Technology \\
  Cambridge, MA 02139 \\
  \texttt{abelhadj@mit.edu}
  \And
  Nicholas H. Nelsen\\
  Department of Mathematics \\
  Cornell University \\
  Ithaca, NY 14853 \\
  and\\
  Oden Institute for Computational Engineering and Sciences \\
  The University of Texas at Austin \\
  Austin, TX 78712 \\
  \texttt{nnelsen@oden.utexas.edu}
}

\ifpdf
\hypersetup{
	pdftitle={One Operator for Many Densities: Amortized Approximation of Conditioning by Neural Operators},
	pdfauthor={P. Tsimpos, E. Calvello, A. Belhadji, and N. H. Nelsen}
}
\fi

\begin{document}

\maketitle

\begin{abstract} 
Probabilistic conditioning is concerned with the identification of a distribution of a random variable $X$ given a random variable $Y$.
It is a cornerstone of scientific and engineering applications where modeling uncertainty is key. 
This problem has traditionally been addressed in machine learning by directly learning the conditional distribution of a fixed joint distribution. 
This paper introduces a novel perspective:
we propose to solve the conditioning problem by identifying a single operator that maps any joint density to its conditional, thus amortizing over joint-conditional pairs.
We establish that the conditioning operator can be approximated to arbitrary accuracy by neural operators. Our proof relies on new results establishing continuity of the conditioning operator over suitable classes of densities.
Finally, we learn the conditioning map for a class of Gaussian mixtures using neural operators, illustrating the promise of our framework.
This work provides the theoretical underpinnings for general-purpose, amortized methods for probabilistic conditioning, such as foundation models for Bayesian inference.
\end{abstract}

\section{Introduction}
Conditioning is a fundamental operation in probabilistic modeling. Given a joint distribution over random variables $(X, Y)$, the conditional distribution $p(X \condbar Y = y)$ for some observed $ y $ is the central object underlying prediction, inference, and decision-making. Modern machine learning systems traditionally solve the problem by learning the conditional distribution of a fixed joint, often from a statistical standpoint.
In this work, we adopt an operator-theoretic perspective on probabilistic conditioning. Specifically, we study the map that takes a joint probability density $\rho$ over $(X,Y)$ and a conditioning variable $y$ to the conditional density $\rho(\slot \condbar y)$. This defines a nonlinear operator
\begin{align*}
	\Psi^\star\colon (\rho, y) \mapsto \rho(\slot \condbar y)
\end{align*}
acting between function spaces.
Under natural regularity assumptions, we show that $ \Psi^\star $ is continuous between suitable metric spaces of probability densities. This seemingly simple observation has a powerful consequence: \emph{conditioning can be universally approximated by neural operators}.
Neural operators are learnable mappings between function spaces that serve as the natural generalization of deep neural networks to the infinite-dimensional setting considered in this paper.
A major benefit of our operator learning perspective is that it enables amortization across joint distributions, which eliminates the need to retrain the neural operator for each new input density.

\paragraph{Contributions.}
Our main contributions are as follows.
\begin{itemize}\item \textbf{Operator-theoretic formulation of conditioning.}  
    We formalize probabilistic conditioning as a nonlinear operator acting on spaces of continuous probability densities and identify two complementary formulations: kernel conditioning and in-context conditioning.

    \item \textbf{Stability of conditioning.}  
    We prove that the conditioning operator is continuous with respect to the supremum norm on suitable subsets of continuous densities. Under additional regularity assumptions, we establish quantitative H\"older stability bounds.

    \item \textbf{Universal approximation by neural operators.}  
    Leveraging the stability results, we show that both formulations of the conditioning operator can be uniformly approximated on compact sets of probability densities by neural operators. 
    In doing so, we propose practical neural operator architectures for conditioning.

    \item \textbf{Unbounded domains and extensions.}  
    We establish a general neural operator universal approximation theorem for continuous operators from $L^1(\R^d)$ to $L^1(\R^{d'})$. We then use it to prove the approximation of the kernel conditioning operator by neural operators over unbounded $L^1$ densities.
    Our analysis also goes beyond uniformly lower bounded marginal densities and characterizes when conditioning admits continuous extensions.

    \item \textbf{Numerical proof-of-concept.} We go beyond theory by numerically demonstrating that \emph{trained} neural operators can accurately implement conditioning in Gaussian mixture test problems.
\end{itemize}

\paragraph{Significance and related work.}

The approximation guarantees established in this paper pave the way for the development of \emph{foundation models for conditioning}: deep learning architectures that, upon being trained on a sufficiently large and diverse dataset, can rapidly perform conditional inference on new, unseen joint probability densities. This has immediate implications for sequential Bayesian inference \cite{stuart2010inverse} tasks such as nonlinear filtering in data assimilation \cite{sanz2023inverse,calvello2025ensemble} or optimal experimental design \cite{huan2024optimal,wu2023fast}, where conditioning is a challenging but key computational primitive.

Our work builds on and connects several lines of research.

\textit{Operator learning.} 
Operator learning was popularized in \cite{bhattacharya2021model,kovachki2021neural,lu2021learning,li2021fourier}. These architectures have been established as state-of-the-art methods for partial differential equations (PDEs)~\cite{li2021fourier}, weather forecasting~\cite{kurth2023fourcastnet,bonev2025fourcastnet}, and complex geometries \cite{zeng2025point,li2023geometry}, to name a few.
One approach called DeepONet \cite{lu2021learning,chen1995universal} approximates operators via a branch–trunk decomposition. Our in-context formulation can be viewed as a continuous analogue of this architecture, where the conditioning variable plays the role of a query input. Another line of work introduced neural operators \cite{kovachki2021neural, kovachki2024operator}, which lift neural networks to function spaces. They now have many theoretical guarantees \cite{lanthaler2025nonlocality,lanthaler2022error,kovachki2024operator,brugiapaglia2026short}. 
Recent work formulates continuum attention as an operator on function spaces, leading to the transformer neural operator \cite{calvello2024continuum}.
These methods have been successfully applied to PDEs \cite{boulle2024mathematical,li2021fourier} and inverse problems \cite{nelsen2026hna,de2025extension}. We extend this framework to probabilistic conditioning operators.

\textit{Conditional generative modeling and conditional kernel mean embeddings.} Conditional GANs \cite{mirza2014conditional}, conditional VAEs \cite{sohn2015learning}, and flow-based models \cite{dinh2017density, lipman2022flow,pierret2026flow} aim to model conditional distributions from data samples only. These approaches are often parametrized in finite-dimensional latent spaces and invoke measure transport \cite{marzouk2017sampling}. In contrast, we study conditioning at the level of function spaces and provide approximation theorems.
A different direction studies conditional kernel mean embeddings \cite{song2009hilbert,li2022optimal} that represent conditional distributions in reproducing kernel Hilbert spaces (RKHS). This motivated work on rigorous operator-theoretic conditional density estimation methods \cite{schuster2020kernel}.
Our operator learning formulation is related, but operates directly in spaces of probability densities rather than RKHS embeddings and amortizes over joint densities without retraining.

\textit{Learning on distribution-valued data and amortized inference.}
Machine learning in metric spaces such as the Wasserstein space of probability distributions \cite{moosmuller2023linear,panaretos2020invitation,lyu2026mvnn} is a contemporary research direction that aligns with our work. 
For distribution-to-distribution regression \cite{szabo2016learning}, measure-theoretic transformers are one prominent method \cite{furuya2025transformers1,furuya2025transformers2,bach2025learning,calvello2026operator,geshkovski2025mathematical,castin2024smooth,biswal2024universal,kawata2026transformers,fraiman2026expressive}. They have been used for amortized unconditional density/score estimation \cite{ilin2026discoformer}, amortized Bayesian inference \cite{whittle2026distribution, gloeckler2024all}, and amortized optimal transport \cite{cole2026context}. These and related deep learning methods have been applied to yield fast non-Gaussian data assimilation algorithms \cite{bach2025learning,calvello2026operator,cui2026amortized,al2025fast,binder2026closed,revach2022kalmannet,tong2026latent}, motivating our work.

\paragraph{Organization.}
The remainder of this paper is organized as follows. \Cref{sec:formulation} introduces the functional-analytic setup and defines the conditioning operators. \Cref{sec:approximation} presents the main approximation results. \Cref{sec:stability} establishes continuity and stability properties. 
\Cref{sec:numerics} presents numerical experiments. Last, \Cref{sec:discussion} concludes with a discussion of implications and future directions.

\section{Formulation}\label{sec:formulation}
Let $D\subset\R^m$ and $E\subset\R^r$ be compact sets with Lipschitz boundaries.
Consider the space of continuous probability densities on $D\times E$ given by
\begin{align*}
    \Cac(D \times E) &\defeq \set{\rho \in C(D \times E)}{\int_{D \times E} \rho(z) \dd{z} = 1 \qa \inf_{z \in D \times E} \rho(z) \geq 0}.
\end{align*}
Write $z\defeq  (x,y)$ for $x \in D$ the \emph{state} or \emph{target} variable and $y \in E$ the \emph{query} or \emph{conditioning} variable.
For each query $y \in E$, we can define the
\emph{data marginal density} by 
\begin{align}\label{eqn:marginal_data_setup}
	m_\rho(y)\defeq \int_D\rho(x,y)\dd{x}.
\end{align}
Finally, for a density $\rho \in \Cac(D \times E)$ we can introduce the collection of queries with non-degenerate marginals by $E_\rho\defeq\set{y\in E}{m_\rho(y)> 0}$,
which allows us to define the \emph{conditional density}
\begin{align}\label{eqn:conditional_pdf}
	\rho(x\condbar y)\defeq \frac{\rho(x,y)}{m_\rho(y)}
\end{align}
for all $x\in D$ and $y \in E_\rho$.
For reasons that will soon become clear we also need to introduce sets of queries with ``well-behaved'' data marginals.
For a strictly positive number $\delta > 0$, let
\begin{align*}
	\cX_\delta\defeq \set{\rho \in \Cac(D\times E)}{\inf_{y\in E} m_\rho(y)\geq \delta}.
\end{align*}
Writing $\sfd_{C(D\times E)}$ and $\sfd_{C(D)}$ for the metrics induced by the supremum norms on $C(D\times E)$ and $C(D)$, respectively, we can constuct metric spaces
\begin{equation}\label{eq:metric_spaces_M_delta_Y}
	\cM_\delta \defeq \bigl(\cX_\delta, \sfd_{C(D\times E)}\bigr)
\qa \cM_\delta^E \defeq  \bigl(\cX_\delta \times E, \sfd_{C(D\times E)\times E}\bigr)
\end{equation} 
with $\sfd_{C(D\times E)\times E}$ induced by the \emph{product norm} $\norm{(\rho,y)}_{C(D\times E)\times E}\defeq\max\bigl(\norm{\rho}_{C(D\times E)}, \| y \| \bigr)$, 
where $ \|\slot\| $ is the Euclidean norm on $ \R^r\supset E $.
Note that this is the natural norm corresponding to the direct sum of Banach spaces $ C(D\times E) \oplus \R^r $. \Cref{lem:completeness-of-metric-spaces} establishes that \eqref{eq:metric_spaces_M_delta_Y} are indeed metric spaces. 

Our goal in this section is to construct a well-behaved conditioning operator, that is, a mapping that takes any density $\rho \in \Cac(D \times E)$ as input and returns its conditional. There exist at least two natural ways to do this: one can either input a $\rho$ together with a query $y$ and output the conditional function $\rho(\slot\condbar y)$, or one can input $\rho$ alone and output the entire kernel function $\kappa_\rho(\slot,\slot)=\rho(\slot\condbar\slot)$. As we will see, both definitions are possible and intimately linked.

\subsection{The in-context conditioning operator}
First, we define the \emph{in-context conditioning operator} $\Psi^\star$ that takes a density $\rho$ and a query $y$ to produce the conditional density $\rho(\slot\condbar y)$.
Noting that for $\rho \in \cX_\delta$ we have $E_\rho = E$, we can define 
\begin{align}\label{eqn:operator_incontext}
	\begin{split}
		\Psi^\star\colon \domain(\Psi^\star)&\to C(D)\\
		(\rho,y)&\mapsto \rho(\slot\condbar y)
	\end{split}
\end{align}
with domain $\domain(\Psi^\star)\defeq \cX_{\delta}\times E$.
We equip $ \domain(\Psi^\star) $ with the metric $\sfd_{C(D\times E) \times E)}$ to obtain the metric space $\cM^E_\delta$, see \eqref{eq:metric_spaces_M_delta_Y}.

\subsection{The kernel conditioning operator}
Now we define the \emph{kernel conditioning operator} $\cG^\star$ that takes a density $\rho$ and produces the kernel function $(x,y) \mapsto \rho(x\condbar y)$.
To do so, we first define the \emph{kernel function} for any $ \rho\in\Cac(D\times E) $ by
\begin{align}\label{eqn:kernel}
	\begin{split}
		\kappa_\rho\colon D\times E_\rho&\to \R_+ \\
		(x,y)&\mapsto \rho(x\condbar y).
	\end{split}
\end{align}
We can now define the \emph{kernel conditioning operator} by
\begin{align}\label{eqn:operator_kernel}
	\begin{split}
		\cG^\star\colon \domain(\cG^\star) &\to C(D\times E)\\
		\rho&\mapsto \kappa_\rho,
	\end{split}
\end{align}
with domain $\domain(\cG^\star)\defeq\cX_{\delta}$.
We equip $ \domain(\cG^\star) $ with the metric induced by the supremum norm to obtain the metric space $\cM_\delta$ \eqref{eq:metric_spaces_M_delta_Y}. Last, we note that $\Psi^\star$ and $\cG^\star$ are related by the \emph{exponential map}; see \cref{rk:exponential-map}.

\section{Main results}\label{sec:approximation}
This section presents the central results of the paper: universal approximation of the conditioning operators $ \cG^\star $ and $ \Psi^\star $ by \emph{neural operators}.
First, we define neural operators in \cref{subsec:neural-operators} as learnable mappings between function spaces, implemented by neural networks.
Then in \cref{subsec:universal-approximation}, we discuss the first set of universal approximation results. These establish that $ \cG^\star $ and $ \Psi^\star $ can be approximated to arbitrary accuracy by neural operators, uniformly over compact sets of continuous densities; examples of the latter include balls of $\alpha$-H\"older densities.   
Last, in \cref{sec:unbounded}, we prove that an extension of $\cG^\star$ defined on all densities over the entire space $\R^{m + r}$ can also be universally approximated by neural operators.
In particular, this delivers approximation guarantees over compact sets of Gaussian mixtures.
We remark that the technical backbone of the approximation theorems in this section are the continuity and stability results that will be presented later in \cref{sec:stability}.

\subsection{Neural operators}\label{subsec:neural-operators}
First, we introduce neural operators following~\cite{lanthaler2025nonlocality}.
Let $\Omega_\textup{in} \subset \R^{d'_\textup{in}}$ and $\Omega_\textup{out} \subset \R^{d'_\textup{out}}$ be Lipschitz domains and $\cF_\textup{in}$ and $ \cF_\textup{out} $ be collections of functions
$\Omega_\textup{in} \to \R^{d_{\textup{in}}}$ and $\Omega_\textup{out} \to \R^{d_\textup{out}}$, respectively. 
Moreover, for some $L\in \N$ and natural numbers $d_1, \ldots, d_L$, we write $\mathcal{F}_i$ for a collection of functions $\Omega_{\textup{in}} \to \R^{d_i}$. We define the following objects.

\begin{definition}[Lifting layer]\label{def:lifting_layer}
	Let $R: \R^{d_\textup{in}' + d_{\textup{in}}} \to \R^{d_0}$ be a neural network.
	A \emph{lifting layer} is a map
	\begin{align*}
		\cR\colon \cF_\textup{in} \to \cF_0 \;\; \textup{with} \;\; \cR(f)(x) = R\bigl(x, f(x)\bigr).
	\end{align*}
\end{definition}
\begin{definition}[Hidden layer]\label{def:hidden_layer}
	The $i$th \emph{hidden layer} is a map of functions
	\begin{align*}
		&\mathcal{L}_\ell \colon \mathcal{F}_{\ell-1} \to \mathcal{F}_\ell, \\
		& \mathcal{L}_\ell(f)(x) \defeq \sigma\left(W_\ell f(x) + b_\ell + \sum_{m=1}^{M_\ell} \left\langle T_{\ell, m} f, \psi_{\ell, m} \right\rangle \phi_{\ell, m}(x)   \right) ,
	\end{align*}
	where $\sigma\colon \R \to \R$ is a nonlinearity applied pointwise, $W_\ell \in \R^{d_\ell \times d_{\ell-1}}$ and $b_\ell \in \R^{d_\ell}$ are weights and biases, and for each $m \in \{1, \ldots, M_\ell\}$, we have: the $T_{\ell, m}\colon \R^{d_{\ell-1}} \to \R^{d_\ell}$ are linear maps, the $\psi_{\ell, m} \in \mathcal{F}_{\ell-1}$ are \emph{encoding functions}, and the $\phi_{\ell, m} \in \mathcal{F}_\ell$ are \emph{decoding functions}.
	Finally, the brackets $\langle \slot, \slot \rangle$ denote the $L^2$ pairing in $\R^{d_{\ell-1}}$.
\end{definition}

The collections of functions $\{T_{\ell, m} f\}_{m=1}^{M_\ell}$, $\{\phi_{\ell, m}\}_{m=1}^{M_\ell}$, and $\{\psi_{\ell, m}\}_{m=1}^{M_\ell}$ could be either learnable or pre-determined. In the latter case, one merely hard-codes them into the architecture, e.g., one could take $\psi_{\ell, m}(x) = \phi_{\ell, m}(x) = e^{\iunit m \cdot x}$ to be Fourier basis functions. In the former case, one may parametrize these functions by linear maps or neural networks and train their parameters.

\begin{definition}[Projection layer]\label{def:projection_layer}
	Let $Q\colon \R^{d_L + d_L'} \to \R^{d_\textup{out}}$ be a neural network. A \emph{projection layer} is a map
	\begin{align*}
		\cQ\colon \cF_L \to \cF_\textup{out} \;\; \textup{with} \;\; \cQ(f)(x) = Q\bigl(x, f(x)\bigr).	
	\end{align*}
\end{definition}
\begin{definition}[Neural operator]\label{def:neural_operator}
	A \emph{neural operator (NO)} of depth $ L > 0 $ is a map of the form
	\begin{align*}
		\Psi\colon \cF_\textup{in} &\to \cF_\textup{out}\\
		f &\mapsto \bigl(\cQ\circ \mathcal{L}_L\circ \cdots \circ \mathcal{L}_1\circ \mathcal{R}\bigr)(f),
	\end{align*}
	where $\cR$ is a lifting layer as in \cref{def:lifting_layer}, each $\mathcal{L}_\ell$ is a hidden layer as in \cref{def:hidden_layer}, and $\cQ$ is a projection layer as in \cref{def:projection_layer}.
\end{definition}
\Cref{def:neural_operator} encompasses popular NOs such as Fourier neural operators~\cite{li2021fourier} and DeepONets~\cite{lu2021learning}. Key to our results is a generalization of NOs to mixed function-vector inputs \cite{huang2025operator,bhattacharya2025learning}. To this end, we require an additional layer type.
\begin{definition}[Vector-to-function layer]\label{def:vector-to-function}
	Fix $d_\textup{vec} \in \N$ and let $\Omega_{d_\textup{vec}} \subset \R^{d_\textup{vec}}$ be a subset. A \emph{vector-to-function layer} is a map of the form
	\begin{equation*}
		V\colon \Omega_{d_\textup{vec}} \to \cF_\textup{in} \;\; \textup{with} \;\; V(z)(x) \defeq h(x) \psi(z) ,
	\end{equation*}
	where $h \in \cF_\textup{in}$ and $\psi\colon \R^{d_\textup{vec}} \to \R$ is a neural network.
\end{definition}

As above, the function $h$ can be either learnable or pre-determined. We also accept the trivial linear neural network $\psi(z) = z$ as a valid choice for $\psi$.

\begin{definition}[Augmented neural operator]\label{def:augmented_neural_operator}
	An \emph{augmented neural operator} (AugNO) of depth $L > 0$ is a map of the form
	\begin{align*}
		\tilde \Psi\colon  \cF_{\textup{in}} \times \Omega_{d_\textup{vec}} &\to \cF_\textup{out}\\
		(f, z) &\mapsto \Psi\bigl(f \oplus V(z)\bigr),
	\end{align*}
	where $\Psi$ is a NO of depth $L$ with input space $\mathcal{F}_{\textup{in}}$ consisting of functions $\Omega_\textup{in} \to \R^{d_{\textup{in}} + d_{\textup{in}}}$ and all remaining parameters as in \cref{def:neural_operator}, $V$ is a vector-to-function layer as in \cref{def:vector-to-function}, and
	$f \oplus V(z)$ denotes the concatenation of the maps $f$ and $V(z)$, i.e., $(f \oplus V(z))(x) = (f(x), V(z)(x))$.
\end{definition}

\subsection{Universal approximation }\label{subsec:universal-approximation}
Our first main result is a universal approximation theorem for the kernel conditioning operator $ \cG^\star $. 
To prove it, we combine the stability results from \cref{sec:stability} with the Dugundji extension theorem~\cite{dugundji1951extension} and a universal approximation theorem for neural operators \cite[Thm.~1, p.~267]{lanthaler2025nonlocality}. All proofs can be found in \cref{app-subsec:proofs_universal_approximation}.
\begin{restatable}[Approximation of $\cG^\star$]{theorem}{thmApproxKernel}\label{thm:approx_kernel}
Fix
	$\delta > 0$ and let $K\subseteq \cX_{\delta}$ be any compact set.
	For any $\ep > 0$, there exists a neural operator
	$\Psi$ with input and output spaces $\cF_\textup{in} = \cF_\textup{out} = C(D\times E) $ such that
\begin{equation}\label{eqn:approx_kernel}
		\sup_{\rho\in K}\norm{\cG^\star(\rho)-\Psi(\rho)}_{C(D\times E)}\leq \ep.
	\end{equation}
\end{restatable}

Notice that the above approximation result holds uniformly over the set of densities $K$. 
That is, there exists a single neural operator that approximates conditioning on the entire compact set $K$ with arbitrarily high accuracy.
However, this engenders the question: what subsets $K$ of $\cX_\delta$ are compact? A representative example is given in \cref{cor:holder_markov}.

We now move on to our second main result.
It is a universal approximation theorem for the in-context conditioning operator, paralleling \cref{thm:approx_kernel}.
\begin{restatable}[Approximation of $ \Psi^\star $]{theorem}{thmApproxIncontext}
\label{thm:approx_incontext}
	For any $\delta > 0$, any compact set $K\subseteq \cX_{\delta}$, and any $\ep > 0$, there exists an augmented neural operator $\widetilde{\Psi}$ with input space $F_\textup{in} = C(D\times E)$ and $\Omega_{d_\textup{vec}} = E$ as well as output space $F_\textup{out} = C(D)$, as in \cref{def:augmented_neural_operator},
	such that
	\begin{align}\label{eqn:approx_incontext}
		\sup_{(\rho,y)\in K\times E}\norm{\Psi^\star(\rho,y)-\widetilde{\Psi}(\rho,y)}_{C(D)}\leq \ep.
	\end{align}
\end{restatable}

The proof of this result extends that of \cref{thm:approx_kernel}. 
The central difficulty is the fact that $\Psi^\star$ takes as input both a function and a vector. To address this, we \emph{lift} $\Psi^\star$ to the operator $\tilde \Psi^\star$ by identifying points $y \in \E$ with constant functions $c_y(z) = y$ for all $z \in \Rd$. We then use \cref{thm:continuity_incontext}, establishing the continuity of $\Psi^\star$, together with the Dugundji extension theorem~\cite{dugundji1951extension} and~\cite[Thm.~1]{lanthaler2025nonlocality} to obtain a neural operator that approximates $\tilde \Psi^\star$.

We close this subsection by giving a concrete example of a set of densities $K$ that is compact in $\cX_\delta$.
\begin{restatable}[Approximation on H\"older compact sets]
{corollary}{corHolderMarkov}\label{cor:holder_markov}
	Fix $ \al\in (0,1]$,  $R  > 0$, and $\delta > 0$. Define
	\begin{align}
		\cH^\alpha_R\defeq\set{f\in C^{0,\al}(D\times E)}{\norm{f}_{C^{0,\al}(D\times E)}\leq R} \qa \mathcal{S} \defeq \cX_\delta \cap \cH^\alpha_R .
	\end{align}
	For any $ \ep>0 $, there exist a neural operator $ \Psi\colon C(D\times E)\to C(D\times E) $ and an augmented neural operator $ \widetilde{\Psi}\colon C(D\times E)\times E\to C(D) $ such that
	\begin{align*}\sup_{\rho\in \mathcal{S}}\norm{\cG^\star(\rho)-\Psi(\rho)}_{C(D\times E)}\leq \ep \qa
		\sup_{(\rho,y)\in \mathcal{S}\times E}\norm{\Psi^\star(\rho,y)-\widetilde{\Psi}(\rho,y)}_{C(D)}\leq \ep.
	\end{align*}
\end{restatable}
The proof combines a classical compact embedding theorem for H\"older spaces with the preceding two universal approximation theorems.
We note that H\"older densities form a rather large class.
Indeed, any Lipschitz density on a compact set is, in particular, $\alpha$-H\"older for any $\alpha \in (0, 1]$. Moreover, rational power functions $x \mapsto x^\alpha$ are $\alpha$-H\"older, for any $\alpha \in (0,1)$, whilst not being Lipschitz near the origin \cite[Chapter 4.1]{gilbarg2001elliptic}. Finally, sample paths of a Brownian motion are locally $\alpha$-H\"older for any $\alpha \in (0, \frac{1}{2})$ \cite[Chapter 2.2]{le2016brownian}.

\subsection{Unbounded domains}\label{sec:unbounded}
In this section, we extend the kernel conditioning operator $\cG^\star$ to densities defined over the entire space $\R^{m+r}$. To do so, we introduce some additional machinery.
For any $n \in \N$, write $\Lac^1(\R^n)$ for the space of all probability density functions. Further recall the notation $\essinf$ for the \emph{essential infimum} of a function in $\R^n$, i.e., the greatest lower bound that holds almost everywhere; see \cref{app-sec:notation}.
Now fix a bounded open set $B\subset\R^r$ and some $\delta>0$. Paralleling the construction of $\cX_\delta$, define
\begin{align}\label{eqn:defn_domain_set_unbounded}
    \cU_{\delta,B}\equiv \cU_\delta\defeq 
        \set{f \in \Lac^1(\R^{m}\times\R^r)}{\essinf_{y\in B} m_f(y)\geq \delta},
\end{align}
where $m_f(y)\defeq \int_{\R^m}f(x,y)\dd{x}$. 
Note that $B$ must be bounded, otherwise $\rho$ would not be a probability density because $\int \rho(x,y) \, dx \, dy = \int m_\rho(y) \, dy \geq \delta \cdot \infty$.
Now for $ \rho\in\Lac^1(\R^m\times \R^r) $, we extend $\kappa_\rho$ and $\cG^\star$ in the natural way.
That is, we let $\kappa_\rho \colon \R^m \times B \to \R_+$ with $\kappa_
\rho(x,y) \defeq \rho(x \condbar y) $ and\footnote{Note that  $\kappa_\rho\in L^1(\R^m\times B)$ due to the uniform lower bound on the marginals.} $\cG^\star\colon \cU_\delta \to L^1(\R^m \times B)$ with $\cG^\star(\rho) \defeq \kappa_\rho$.
Now, we define a class of neural operators that map between function spaces on unbounded domains.

\begin{definition}[FullNO]\label{eqn:defn_neural_op_unbounded}
For any $n \in N$ and any measurable set $\Omega\subset \R^n$, define the \emph{zero extension operator} $\cZ_{\Omega}\colon L^1(\Omega) \to L^1(\R^n)$ by 
$\cZ_\Omega(f)(x) = f(x)$ for $x \in \Omega$ and $\cZ_\Omega(f)(x) = 0$ for $x \not\in \Omega$. We also define the \emph{restriction operator} $\cR_\Omega\colon L^1(\R^n)\to L^1(\Omega)$ by $f\mapsto f|_\Omega$.
Fix $d$ and $ d'$ in $ \N $ and measurable sets $\Omega \subset \R^{d}$ and $\Omega' \subset \R^{d'}$. A \emph{full space neural operator (FullNO)} between $L^1(\R^{d})$ and $L^1(\R^{d'})$ with truncation sets $\Omega$ and $\Omega'$ is a map $\Psi\colon L^1(\R^{d})\to L^1(\R^{d'})$ of the form
\begin{align}\label{eqn:defn_full_no}
	\Psi = \cZ_{\Omega'} \circ \Psi_0 \circ S \circ \cR_{\Omega},
\end{align}
where $\Psi\colon L^1(\Omega)\to L^1(\Omega')$ is a NO in the sense of \cref{def:neural_operator} and $S\colon \R \to \R$ is a shallow ReLU neural network acting on an input function $f$ pointwise, i.e., $S(f)(x) = S(f(x))$ for each $x \in \Omega$.
\end{definition}

To prove a universal approximation theorem for $\cG^\star$ over $\R^{m + r}$, we combine \cref{lem:global_lip_kernel_operator_unbounded} with a novel whole space universal approximation theorem. Specifically, we prove that FullNOs can universally approximate continuous operators $ F^\star $ between $ L^1(\R^n) $ spaces; the estimate is uniform over an $ L^1 $-compact set $K$ of functions. The proof consists of five main steps. First, we cut off functions in $K$ at some $M > 0$ to obtain a set $K_M$ bounded in $L^\infty$. In the second and third steps, we show that we can introduce truncations sets $\Omega$ and $\Omega'$ and a map $\Psi_0\colon L^1(\Omega) \to L^1(\Omega')$ such that $\Psi_0$ is a good approximation of $F^\star$ uniformly over $K_M$. In the fourth step, we invoke an existing universal approximation theorem to obtain a NO $\Psi_{00}$ that approximates $\Psi_0$ well over $K_M$. In the last step, we compose $\Psi_{00}$ with appropriate truncation and zero extension operators to obtain a full space NO $\Psi$ that approximates $F^\star$ well over $K$. The details are in \cref{app-subsec:proofs_universal_approximation}.
\begin{restatable}[Universal approximation on whole space]{theorem}{thmUaMarkovUnbounded}\label{thm:ua_markov_unbounded}
	Let $F^\star\colon L^1(\R^d)\to L^1(\R^{d'})$ be continuous. Let $\ep>0$ be arbitrary. Let $K\subset L^1(\R^d)$ be any compact set. Then there exist $R>0$, $R'>0$, and a full space neural operator $\Psi\colon L^1(\R^d)\to L^1(\R^{d'})$ with $\Omega\defeq(-R,R)^d$ and $\Omega'\defeq(-R',R')^{d'}$ truncation sets as in \eqref{eqn:defn_neural_op_unbounded} such that
	\begin{align}
		\sup_{\rho\in K}\norm{F^\star(\rho)-\Psi(\rho)}_{L^1(\R^{d'})}\leq\ep.
	\end{align}
\end{restatable}
\begin{restatable}[Approximation of $\cG^\star$ on whole space]{corollary}{corUaMarkovUnboundedConditioning}\label{cor:ua_markov_unbounded_conditioning}
	Let $\ep>0$ and $\delta>0$ be arbitrary. Let $B\subset\R^r$ be bounded, open, and measurable. Let $\cU_\delta=\cU_{\delta, B}$ be as in \eqref{eqn:defn_domain_set_unbounded}.  Let $K\subset\cU_\delta$ be compact in $L^1(\R^m\times\R^r)$. Then there exist $R>0$ and a full space neural operator $\Psi\colon L^1(\R^m\times\R^r)\to L^1(\R^m\times\R^r)$ depending on $\Omega=\Omega'\defeq (-R,R)^{m+r}$ as in \eqref{eqn:defn_neural_op_unbounded} such that
	\begin{align}\label{eqn:ua_markov_unbounded_conditioning}
		\sup_{\rho\in K}\norm{\cG^\star(\rho)-\Psi(\rho)}_{L^1(\R^m\times B)}\leq\ep.
	\end{align}
\end{restatable}

The above theorem allows us to obtain an approximation over sets of densities $K$ that embed compactly in $\cU_\delta$. One such example is given by bounded sets in \emph{weighted Sobolev spaces}.\footnote{
    Below we are using notation $L^1_\alpha(\R^{m+r}) \defeq \set{f \in L^1(\R^{m+r})}{\int |f(x) | \, (1 + \|x\|^\alpha) \dd{x} < \infty }$ for $\alpha > 0$.
}
Indeed, if $K$ is a bounded subset of $W^{1,1}_\alpha(\R^{m+r}) \defeq W^{1,1}(\R^{m+r}) \cap L^1_\alpha(\R^{m+r}) $ then by~\cite[Corollary 4.27]{brezis2011functional} $K \cap \cU_\delta \subset \cU_\delta$ is compact. In particular, the density of any Gaussian mixture lies in $W^{1,1}_\alpha(\R^{m+r})$ for any $\alpha > 0$. Thus, we obtain approximation of the conditioning operator over sets of Gaussian mixtures with bounded means and covariance matrices.

\section{Continuity of conditioning}\label{sec:stability}
The key to establishing the universal approximation arguments of \cref{sec:approximation} is to show that the conditioning operators $\cG^\star$ and $\Psi^\star$ are continuous. \Cref{subsec:stability-results} establishes stability results for the conditioning operators, in particular implying that they are continuous.
In \cref{sec:removing_unif_lower_bound}, the conditioning operators are continuously extended beyond $\cX_\delta$ to the space of \emph{positive densities}.
The reader can refer to \cref{sec:formulation} for a reminder on the domains and codomains of $\cG^\star$ and $\Psi^\star$  as well as for the topologies used.
All proofs may be found in \cref{app-subsec:proofs_stability}.
\subsection{Stability results}\label{subsec:stability-results}
First, we establish the local Lipschitz stability for the kernel conditioning operator $\cG^\star$ on the metric space $\cM_\delta$ defined in \eqref{eq:metric_spaces_M_delta_Y}. That is, $\cG^\star$ is Lipschitz continuous in a neighborhood of any point in $\cX_\delta$ equipped with the topology of uniform convergence. 
\begin{restatable}[$ \cG^\star $ is locally Lipschitz]{lemma}{lemLocalLipKernelOperator}
\label{lem:local_lip_kernel_operator}
	For any $ p\in\cX_{\delta} $ and $ q\in\cX_{\delta} $, it holds that
	\begin{align}
		\norm{\cG^\star(p)-\cG^\star(q)}_{C(D\times E)}\leq
		\frac{1}{\delta} \biggl(1+\frac{\abs{D}\min(\norm{p}_{C(D\times E)},\norm{q}_{C(D\times E)}}{\delta}\biggr)\norm{p-q}_{C(D\times E)}.
	\end{align}
\end{restatable}
The result is elementary, relying on the uniform lower bound $\delta > 0$ on densities in $\cX_\delta$ to avoid a division by zero in the definition of the conditional density.
The following corollary delivering global Lipschitz stability on a restricted domain is immediate; cf.~\cref{rmk:uniform_bound}.
\begin{corollary}[$ \cG^\star $ is Lipschitz on balls]\label{cor:local_lip_kernel_operator}
Fix $R > 0$ and define
	\begin{equation*}
		\mathcal{B}_R \defeq \set{f \in C(D\times E)}{\|f\|_{C(D\times E)} \leq R}.
	\end{equation*}
	Then, for any $p$ and $ q $ in $ \mathcal{B}_R \cap \cX_\delta$, it holds that
	\begin{align}
		\norm{\cG^\star(p)-\cG^\star(q)}_{C(D\times E)}\leq
		\frac{1}{\delta} \biggl(1+\frac{\abs{D}  R}{\delta}\biggr)\norm{p-q}_{C(D\times E)}.
	\end{align}
\end{corollary}
Finally, we can get an analogous result to \cref{lem:local_lip_kernel_operator} by changing the underlying function spaces. 
\begin{restatable}[Whole space Lipschitz]
{lemma}{lemGlobalLipKernelOperatorUnbounded}\label{lem:global_lip_kernel_operator_unbounded}
	Let $\cU_\delta$ be as in \eqref{eqn:defn_domain_set_unbounded}. 
    For any $ p\in\cU_{\delta} $ and $ q\in\cU_{\delta} $, it holds that
	\begin{align}
		\norm{\cG^\star(p)-\cG^\star(q)}_{L^1(\R^m\times B)}\leq
		\frac{2}{\delta} \norm{p-q}_{L^1(\R^m\times\R^r)}.
	\end{align}
\end{restatable}

We now turn our attention to the in-context conditioning operator $\Psi^\star$. The analysis of this operator is slightly more complicated than that of $ \cG^\star $. 
The next result shows that $\Psi^\star$ is continuous.

\begin{restatable}[Continuity of $ \Psi^\star $]{theorem}{thmContinuityIncontext}
\label{thm:continuity_incontext}
	The map $ \Psi^\star $ in \eqref{eqn:operator_incontext} is continuous.
\end{restatable}
The main idea of the proof is to take a convergent sequence $(p_n,y_n)\to (p,y)$ in $\cX_\delta\times E$ and use a triangle inequality to bound the error $\| \Psi^\star(p_n,y_n)-\Psi^\star(p,y)\|_{C(D)}$ by the sum of two terms: $\| \Psi^\star(p_n,y_n)-\Psi^\star(p,y_n)\|_{C(D)}$ and $\|\Psi^\star(p,y_n)-\Psi^\star(p,y)\|_{C(D)}$. The first term can be made small by taking a supremum over all $y \in E$ and using \cref{lem:local_lip_kernel_operator}; the the second term can be made small by an elementary argument, using the uniform lower bound $\delta > 0$ as well as the fact that continuous densities $p_n$ on the compact set $D\times E$ are uniformly continuous and hence $p_n \to p$ implies that marginals converge $m_{p_n} \to m_p$ as $ n\to\infty $.

We can upgrade the qualitative continuity result of \cref{thm:continuity_incontext} to a quantitative H\"older stability estimate for $ \Psi^\star $ by further restricting its domain. This is the content of the following corollary.

\begin{restatable}[H\"older stability of $ \Psi^\star $ on balls]{corollary}{corHolderStabilityIncontext}
\label{cor:holder_stability_incontext}
	Let $ \al\in (0,1] $. Fix $ R>0 $ and define
	\begin{align*}
		\cH_R^\alpha\defeq\set{f\in C^{0,\al}(D\times E)}{\norm{f}_{C^{0,\al}(D\times E)}\leq R}.
	\end{align*}
 	For any $ (p,y) $ and $ (q, y') $ belonging to $ (\cX_{\delta}\cap \cH_R^\alpha)\times E $, it holds that
	\begin{align}\label{eqn:holder_stability_incontext}
		\norm{\Psi^\star(p,y)-\Psi^\star(q,y')}_{C(D)}\leq \frac{(1+2R)}{\delta}\biggl(1+\frac{\abs{D}R}{\delta}\biggr)\norm[\big]{(p,y)-(q,y')}^\al_{C(D\times E)\times E}\,.
	\end{align}
\end{restatable}
We note that the theorem remains true under slightly relaxed assumptions.
Namely, one can ask for a continuous $f \in C(D \times E)$ and for $(x,y) \in D\times E$, ask for $f$ to be $\alpha$-H\"older continuous in $y$ uniformly in $x$. See \cref{remark:relaxed-Holder-cont} for a more precise statement.

\subsection{Beyond uniform lower bounds}\label{sec:removing_unif_lower_bound}
First, let $\mathcal{X}_+ \defeq \bigcup_{\delta > 0} \mathcal{X}_\delta $ be the union of all $\cX_\delta$ for $\delta > 0$.
It is not hard to see that this is exactly the set of densities with everywhere positive marginals, i.e., $\mathcal{X}_+ = \set{\rho \in \Cac(D \times E)}{m_\rho > 0}$, see \cref{lemma:representation_of_X_plus}.
We can use this observation to show that $ \cG^\star $ and $ \Psi^\star $ extend continuously to $\mathcal{X}_+$.
\begin{restatable}[Continuous extension of conditioning operators]{theorem}{propExtensionOfOperators}
	\label{prop:extension-of-operators}
	The conditioning operators $ \cG^\star $ and $ \Psi^\star $ admit unique continuous extensions to the domain $ \mathcal{X}_+$ and $ \mathcal{X}_+ \times E $, respectively.
\end{restatable}
The key proof idea is to recognize that any convergent sequence $\rho_n \to \rho$ in $\mathcal{X}_+$ must eventually be contained in some $\cX_\delta$ for $\delta > 0$. Then, we can use the continuity of $\cG^\star$ and $\Psi^\star$ on each $\cX_\delta$, shown in \cref{sec:stability}, to obtain the convergence $\Psi^\star(\rho_n,y_n) \to \Psi^\star(\rho,y)$ as $ n\to\infty $ and similarly for $\cG^\star$.

Finally, we can use the density of $\mathcal{X}_+$ in $C_{\textup{ac}}(D \times E)$, established in \cref{prop:density-of-X}, to give an explicit expression for any extension of $\Psi^\star$ beyond $\mathcal{X}_+$ in terms of convolutions.

\begin{restatable}[Extension as a limit]{theorem}{thmExtensionOfPsi}
	\label{thm:extension-of-Psi}
	Any continuous extension of $ \Psi^\star $ to a domain $S_1 \times S_2$ satisfying
	\begin{equation*}
		\mathcal{X}_+ \times E \subseteq S_1 \times S_2 \subseteq C_{\textup{ac}}(D \times E) \times E
	\end{equation*}
	must be of the form
	\begin{equation*}
		\Psi^\star(\rho,y)(x) = \lim_{\epsilon \to 0} \frac{\rho \ast \chi_\epsilon(x,y)}{m_{\rho} \ast m_{\chi_\epsilon}(y)},
	\end{equation*}
	where $\chi_\epsilon$ is the density of a centered Gaussian measure with covariance matrix $\epsilon I$.
\end{restatable}
\begin{example}[Products]
	Consider the extended domain $S_1 \times S_2$ given by taking $S_2 \defeq E$ and 
	\begin{equation*}
		S_1 \defeq \mathcal{X}_+ \cup \set{\rho\colon (x,y) \mapsto f(x) \, g(y)}{ f \in C(D)\qa g \in C(E) } .
	\end{equation*}
	That is, we add to $\mathcal{X}_+$ all product densities, irrespective of positivity. We claim that for all $(\rho, y) \in S_1 \times S_2$, the limit in \cref{thm:extension-of-Psi} exists and equals $ \Psi^\star(\rho,y)(x) = f(x)$ for each $x$.
Refer to \cref{lemma:extension-of-Psi-to-indep-densitites} for a proof.
    Thus, we recover the well-known result that conditioning a product distribution by one of its marginals merely removes that marginal. In probabilistic language, if the random variables $(X, Y)$ have joint density $\rho(x,y) = f(x) \, g(y)$, then $X \condbar Y$ is distributed according to $X$ which has density $f$.

    Finally, we note that $\Psi^\star$ cannot be extended continuously to the full space of probability measures over $\R^{m + r}$ when equipped with the weaker topology of \emph{weak} or \emph{narrow} convergence.
One can show this by conditioning a sequence of two component mixtures of Dirac masses; see \cref{example:conditioning-cannot-be-cont-in-weak-conv} for details. 
    This suggests that departing from the setting of probability measures with Lebesgue densities poses fundamental difficulties even for the simplest of examples.
\end{example}

\section{Numerical experiments}\label{sec:numerics}

We train a Fourier neural operator (FNO) \citep{li2021fourier} and a transformer neural operator (TNO) \citep{calvello2024continuum} to approximate the kernel conditioning operator $\cG^\star$. We consider $K$-component Gaussian mixture joint distributions, for which a closed-form conditional Markov kernel may be found. \Cref{subsec:description-of-experiments} discusses our experimental setup, \cref{subsec:data-generation} discusses the data generation process, and \cref{subsec:architectural} discusses architectural and training details. In \Cref{tab:L1_results}, we report the results obtained by applying a trained TNO and FNO to a test set not used for training; \Cref{fig:TNO} illustrates the result of TNO on the median error sample, for the experiment corresponding to $K=3$.
The operator learning framework discussed in this paper leads to accuracy benefits compared to a traditional plug-in estimation approach, involving numerical integration for marginalization. This is demonstrated by deploying the trained models on kernel density estimate (KDE) joints in both experimental settings.
\Cref{fig:TNO_corr} showcases, in the context of $K=1$, both the improvement in accuracy given by the TNO over the plug-in estimator in the test set for the KDE setting, as well as the predictions for the median error sample. \Cref{tab:L1_results_kde} summarizes the accuracy results in both KDE setups. In the $K=1$ setting, the learned models significantly outperform the plug-in baseline. In the more challenging $K=3$ setting, learned models achieve lower median error but higher maximum error than the plug-in baseline. We expect more improvement with larger training sets and training data that includes KDEs as input functions. These results suggest that the conditioning operator may be learned with potential accuracy gains to traditional statistical approaches. We note, however, that these experiments are intended as a proof of concept rather than a practical proposal. Their purpose is twofold: first, to numerically illustrate the universal approximation theory developed, confirming that operator learning is a principled approach to kernel conditioning, and second, to motivate further methodological development. Indeed, grid discretization and the required density access pose limitations to high-dimensional scalability. This is discussed further in \cref{sec:discussion}. 

\begin{figure}[tb!]\centering
    \includegraphics[width=0.6\linewidth]{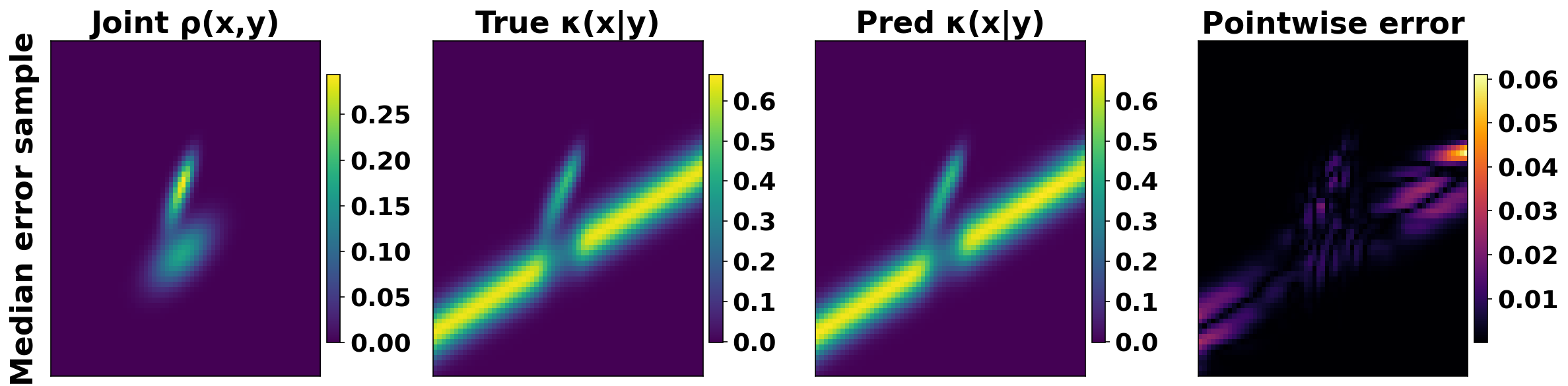}
    \vspace{-1mm}
    \caption{\small Trained TNO prediction on the median relative $L^1$ error sample from the Gaussian mixture test set.}
    \label{fig:TNO}
\end{figure}

\begin{table}[tb!]\centering
\caption{\small Accuracy in relative $L^1$ error achieved by the models on a test set not used for training.}
\label{tab:L1_results}
\resizebox{\textwidth}{!}{\begin{tabular}{l|cc|cc|cc}
\toprule
\textbf{Joint Class} 
& \textbf{FNO Med.} & \textbf{TNO Med.}
& \textbf{FNO Max.} & \textbf{TNO Max.}
& \textbf{FNO Params} & \textbf{TNO Params} \\
\midrule
Corr. Gaussian ($K=1$)  & 0.0004 & 0.0006 & 0.0299 & 0.2683 & 9.6M & 1.6M \\
Gaussian Mixture ($K=3$) & 0.0158 & 0.0251 & 0.7384 & 0.8938 & 9.6M & 3.2M \\
\bottomrule
\end{tabular}
}
\end{table}

\begin{figure}[tb!]
    \centering
    \begin{subfigure}[c]{0.33\textwidth}
        \centering
        \includegraphics[width=0.8\linewidth]{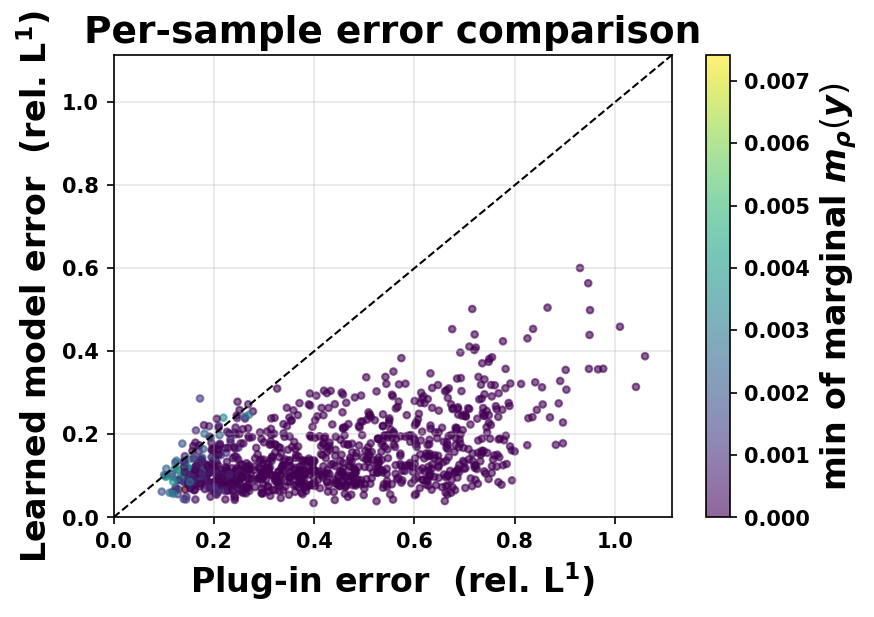}
        \label{fig:plot1}
    \end{subfigure}
    \hfill
    \begin{subfigure}[c]{0.65\textwidth}
        \centering
        \includegraphics[width=\linewidth]{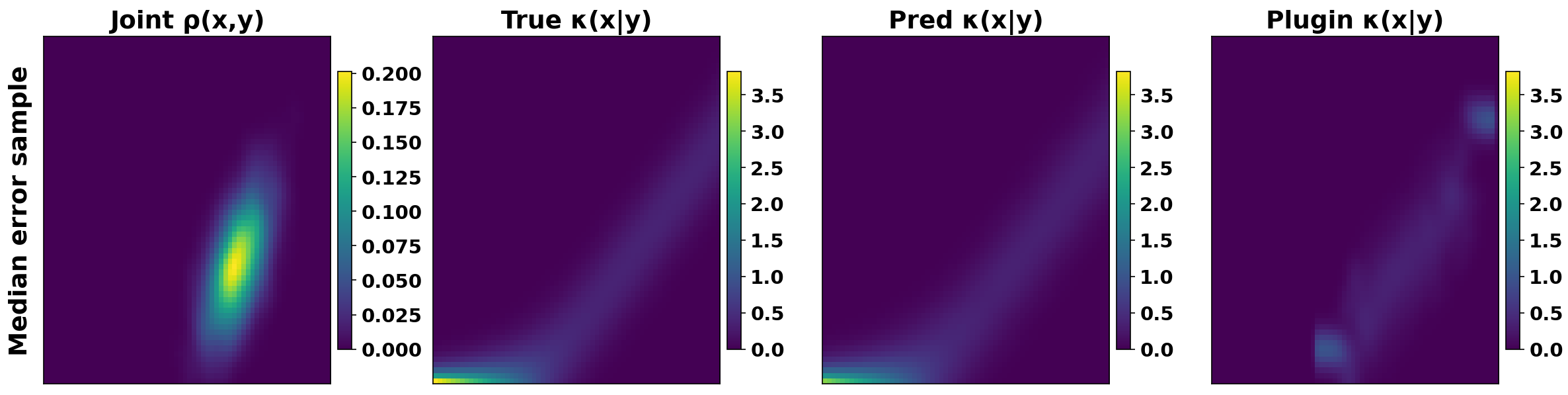}
        \label{fig:plot2}
    \end{subfigure}
    \caption{\small Per-sample relative $L^1$ test error comparison between plug-in estimator and learned model on the correlated Gaussian KDE test set (left). Predictions on the learned model median test error sample (right).}
    \label{fig:TNO_corr}
\end{figure}

\begin{table}[tb!]\centering
\caption{\small Accuracy in relative $L^1$ error achieved by models and the plugin estimator on a KDE test set.}
\label{tab:L1_results_kde}
\resizebox{\textwidth}{!}{\begin{tabular}{l|ccc|ccc}
\toprule
\textbf{Joint Input} 
& \textbf{FNO Med.} & \textbf{TNO Med.} & \textbf{Plugin Med.}
& \textbf{FNO Max.} & \textbf{TNO Max.} & \textbf{Plugin Max.}  \\
\midrule
KDE Corr. Gaussian    & 0.1183 & 0.1244 & 0.4012 & 0.5822 & 0.6004 & 1.0597 \\
KDE Gaussian Mixture  & 0.4829 & 0.4446 & 0.5079 & 1.1455 & 1.2511 & 0.9993 \\
\bottomrule
\end{tabular}
}
\end{table}

\section{Discussion}\label{sec:discussion}
In this paper, we introduce an operator learning framework for the conditioning of probability densities.
Namely, we identify metric spaces of functions such that the maps $\rho \mapsto \kappa_\rho$, the \emph{kernel conditioning operator}, and $(\rho, y) \mapsto \rho(\slot\condbar y)$, the \emph{in-context conditioning operator}, are continuous. Continuity allows us to approximate these maps by neural operators up to arbitrary accuracy and uniformly over compact sets of densities. Our results lay the groundwork for a foundation model for probability density conditioning.
That is, a model that has been trained on a large corpus of joint and conditional density pairs and can perform conditional inference on any unseen density within a class of admissible inputs, e.g., a bounded set of H\"older densities over a compact domain.

\textbf{Open questions} include:
\begin{enumerate}
    \item \emph{Development of scalable algorithms for operator learning in high dimensions}:
    Since neural operators receive and return density functions, they must be combined with methods that allow for efficient function representation in high dimensions, beyond simple grid evaluation.
    \item \emph{Learning from samples}: Commonly one has \emph{sample access} as opposed to \emph{density access} to the probability distribution of interest. Thus, classical density estimation or modern generative methods can be combined with neural operators to obtain pipelines for conditioning from samples.
\item \emph{Conditioning of generic probability measures}: With an eye towards further universal approximation theorems, we wish to study the continuity of the conditioning operators over classes of generic probability measures, e.g., discrete measures or singular measures.
\end{enumerate}

\begin{ack}
PT acknowledges support from the US Air Force Office of Scientific Research (AFOSR) MURI program, under award number FA9550-20-1-0397, and from the ExxonMobil Technology and Engineering Company. PT also acknowledges support from the Onassis Foundation PhD Fellowship.
EC acknowledges support from the Department of Defense (DOD) Vannevar Bush Faculty Fellowship (award N00014-22-1-2790) held by Prof. Andrew Stuart; EC also acknowledges support from the Resnick Sustainability Institute at Caltech. 
AB is grateful for support from the US Department of Energy (DOE), Office of Advanced Scientific Computing Research (ASCR), under awards DE-SC0023188 and DE-SC0026245 held by Prof. Youssef Marzouk. AB also acknowledges support from the ExxonMobil Technology and Engineering Company. 
NHN acknowledges support from a Klarman Fellowship through Cornell University's College of Arts \& Sciences.
The computations presented in this paper were partially conducted on the Resnick High Performance Computing Center, a facility supported by the Resnick Sustainability Institute at the California Institute of Technology.
\end{ack}

{
	\bibliographystyle{abbrvnat}

}

\appendix

\clearpage
\begin{center}
	{\bf Technical Appendices and Supplementary Material for:} \\ 
	\inserttitle \\
\end{center}

\section{Notation}\label{app-sec:notation}
This appendix introduces notation from analysis and probability.
\subsection{Analysis}\label{app-subsec:functional_analysis}
Let $Y$ be a normed vector space and let $C(X;Y)$ denote the vector space of continuous functions from $X$ to $Y$. For any $f$, define the supremum norm
\begin{align}\label{eqn:sup_norm}
	\norm{f}_{C(X;Y)}\defeq\sup_{x\in X}\norm{f(x)}_Y.
\end{align}
The norm \eqref{eqn:sup_norm} induces the metric
\begin{align}\label{eqn:sup_metric}
	\sfd_{C(X;Y)}(f,g)\defeq \norm{f-g}_{C(X;Y)}
\end{align}
for any $f$ and $g$. 
When $Y=\R$, we write $C(X;\R)=C(X)$. 
When $ Y $ is a normed vector space, we equip the product space $ X\times Y$ with the norm $ \norm{(x,y)}_{X\times Y}\defeq\max(\norm{x}_X,\norm{y}_Y) $. 
H\"older spaces play an important role in this work. For $ \al\in (0,1] $ and a compact set $ B\subset\R^k $, define the norm
\begin{align}\label{eqn:holder_norm}
	\norm{f}_{C^{0,\al}(B)}\defeq \max\Biggl(\norm{f}_{C(B)}, \sup_{\substack{(x,x')\in B\times B\\x\neq x'}}\frac{\abs{f(x)-f(x')}}{\abs{x-x'}^\al}\Biggr).
\end{align}
Then define
\begin{align}\label{eqn:holder_space}
	C^{0,\al}(B)\defeq \set{f\in C(B)}{\norm{f}_{C^{0,\al}(B)}<\infty}.
\end{align}
Moreover, define the \emph{essential infimum} as follows: fix $n \in \N$ and a measurable function $f\colon \R^n \to \R$.
For $b \in \R$ write $L_b \defeq \set{x}{f(x) < b}$ for a sub-level set of $f$. Then the essential infimum is
\begin{equation*}
    \essinf_{x \in \R^n} f(x) \defeq \sup \set{b \in \R}{ | L_b | = 0} ,
\end{equation*}
where $| \slot | $ indicates the Lebesgue measure of a set.

\subsection{Additional notation}\label{app-subsec:problem_specific_notation}
Let $D\subset\R^m$ and $E\subset\R^r$ be compact sets. Then $(C(D), \norm{\slot}_{C(D)})$ and $(C(D\times E), \norm{\slot}_{C(D\times E)})$ are Banach spaces. For any compact $B\subset\R^k$, define the subset
\begin{align}\label{eqn:ac_set}
	\Cac(B)\defeq\set{\rho\in C(B)}{\int_B\rho(u)\dd{u}=1\qa \inf_{u\in B} \rho(u)\geq 0}\subset C(B).
\end{align}
We view $\Cac(B)$ in \eqref{eqn:ac_set} as the set of continuous Lebesgue probability density functions on $B$. 

\section{Proofs for Section~\ref{sec:formulation}:~\nameref*{sec:formulation}}\label{app-subsec:proofs_formulation}
The following lemma shows that $\cM_{\delta}$ and $\cM_\delta^E$ as defined in \eqref{eq:metric_spaces_M_delta_Y} are complete. 
In other words, $\cX_\delta$ and $\cX_\delta\times E$ are closed in $C(D \times E)$ and $C(D \times E) \times E$, respectively, with respect to the supremum norm and the product norm, discussed beneath \eqref{eq:metric_spaces_M_delta_Y}.
In fact, we only show that $\cM_\delta$ is closed; then $\cM_\delta^E$ is closed as the Cartesian product of closed sets.
\begin{lemma}[Completeness]\label{lem:completeness-of-metric-spaces}
	It holds that $\cM_{\delta}$ as defined in \eqref{eq:metric_spaces_M_delta_Y} is a complete metric spaces.
\end{lemma}
\begin{proof}
    Fix $\delta > 0$ and set
    \begin{equation*}
        \mathcal{L}_\delta \defeq \set{f \in C(D \times E) }{\inf_{x \in D \times E} f(x) \geq \delta} .
    \end{equation*}
    Note that we can write
    \begin{equation*}
        \cX_\delta = \Cac(D \times E) \cap \mathcal{L}_\delta .
    \end{equation*}
    It suffices to show that both sets in the above intersection are closed in $C(D \times E)$ with respect to the topology of uniform convergence.
	First, we show that $\mathcal{L}_\delta$ is sequentially closed in $C(D\times E)$. Let $\{f_n\}_{n\in\N}\subset \mathcal{L}_\delta$ be a sequence such that $f_n\to f \in C(D\times E)$ in $C(D\times E)$ as $ n\to\infty $. Fix $y\in E$. It follows that $f_n(\slot, y)\to f(\slot,y)$ uniformly and hence pointwise in $D$. Since $\sup_{n\in\N}\norm{f_n}_{C(D\times E)}<\infty$, the dominated convergence theorem applied to $\{f_n(\slot,y)\}_{n\in\N}$ gives
	\begin{align*}
		m_f(y)=\int_Df(x,y)\dd{x}&=\liminf_{n\to\infty}\int_Df_n(x,y)\dd{x}\\
		&\geq \liminf_{n\to\infty} \inf_{y'\in E} \int_Df_n(x,y')\dd{x}\\
		&\geq \delta.
	\end{align*}
	Since $y$ is arbitrary, it follows that $ \mathcal{L}_\delta $ is closed.
	
	Next, for any $n \in \N$ and compact $B \subset \R^n$ we show that the set
    \begin{equation*}
        \Cac(B) = \set{f \in C(B)}{\int_B f(x) \, \dd{x} \qa \inf_{x \in B} f(x) \geq 0 } \, .
    \end{equation*}
    is closed in $C(B)$. The assertion of the lemma follows by taking $B=D\times E$. To this end, let $\{\rho_n\}_{n\in\N}\subset \Cac(B)$ be a sequence such that $\rho_n\to\rho \in C(B)$ in $C(B)$ as $ n\to\infty $. We must show that $\rho\in \Cac(B)$. Since $\rho_n$ converges to $\rho$ uniformly, it also converges pointwise. Moreover, there exists $0<c<\infty$ such that $\sup_{n\in\N}\norm{\rho_n}_{C(B)}\leq c$.
	Now fix $u\in B$. Then for all $n\in\N$, it holds that
	\begin{align*}
		\rho(u)\geq \rho_n(u) - \abs{\rho_n(u)-\rho(u)}\geq  - \abs{\rho_n(u)-\rho(u)}.
	\end{align*}
	We deduce that $\rho(u)\geq -\limsup_{n\to\infty}\abs{\rho_n(u)-\rho(u)}=0$ as required. Next, we see that $\rho_n$ is pointwise uniformly bounded in $n$ by $c$. An application of the dominated convergence theorem yields
	\begin{align*}
		\int_B\rho(u)\dd{u}=\lim_{n\to\infty}\int_B\rho_n(u)\dd{u}=1.
	\end{align*}
	This completes the proof.
\end{proof}

\begin{remark}[Uniform $L^\infty$ upper bound]\label{rmk:uniform_bound}
	A similar argument shows that for any $B$ and any $R>0$, the set $\mathcal{B}_R(B)\defeq \set{\rho\in C(B)}{\sup_{x\in B}\abs{\rho(x)}\leq R}$ is closed in $C(B)$. Thus, $\cX_{\delta}\cap \mathcal{B}_R(D\times E)$ is also closed in $ C(D\times E) $. This set is useful because it will be shown that one formulation of distribution conditioning is Lipschitz on $\cX_{\delta}\cap \mathcal{B}_R(D\times E)$.
\end{remark}

\section{Proofs for Section~\ref{sec:approximation}:~\nameref*{sec:approximation}}\label{app-subsec:proofs_universal_approximation}

This appendix collects the proofs of the main results.
\subsection{Proofs for Subsection~\ref{subsec:universal-approximation}:~\nameref*{subsec:universal-approximation}}

\thmApproxKernel*
\begin{proof}
	By \cref{lem:local_lip_kernel_operator}, $ \cG^\star\colon\cX_{\delta}\subset C(D\times E)\to C(D\times E) $ is continuous. Since $ K\subseteq\cX_{\delta} $, the map $ F^\star\defeq \cG^\star|_{K}\colon K\subset C(D\times E)\to C(D\times E)  $ is also continuous. Since $ K $ is compact, $ K $ is closed.
	By the Dugundji extension theorem \cite[Thm.~4.1, p.~357]{dugundji1951extension}, there exists a continuous map $ \widetilde{F}^\star\colon C(D\times E)\to C(D\times E) $ that extends $ F^\star $. An application of \cite[Thm.~1 and Cor.~1, p.~267]{lanthaler2025nonlocality} to $  \widetilde{F}^\star $ delivers a neural operator $ \Psi $ such that \eqref{eqn:approx_kernel} holds.
\end{proof}

\thmApproxIncontext*
\begin{proof}
	For any $y \in E$, consider the constant function
	\begin{align*}
		c_y\colon D \times E \to \R^r , \qw c_y(x,z) = y \qfa (x,z) \in D \times E .
	\end{align*}
	Now define
	\begin{equation*}
		\domain(\widetilde \Psi^\star) \defeq \set{ \rho \oplus c_y}{(\rho, y) \in \domain(\Psi^\star)} \subset C(D\times E; \R^{r + 1}),
	\end{equation*}
	where we are using notation $\oplus$ to denote concatenation of functions, that is\footnote{\label{footnote:concat}
	Here, and for the rest of the proof we identity $C(D\times E) \times C(D \times E; \R^r)$ with $C(D\times E; \R^{r+1})$ via the continuous embedding
	\begin{align*}
		J\colon C(D\times E) \times C(D \times E; \R^r) &\to C(D\times E ; \R^{r+1}) \\
		(\rho, c_y) &\mapsto  \rho \oplus c_y \colon \bigl((x,z) \mapsto (\rho(x,z), c_y(x, z))\bigr) .
	\end{align*}
	}
	\begin{equation*}
		(\rho \oplus c_y)(x,z) \defeq \bigl(\rho(x,z), c_y(x, z)\bigr) .
	\end{equation*}
Now let the lift $\tilde \Psi^\star$ of our operator $\Psi^\star$ be given by
	\begin{align*}
		\tilde \Psi^\star\colon  \domain(\tilde \Psi^\star)& \to C(D) \\
		\rho \oplus c_y &\mapsto \Psi^\star(\rho,y) .
	\end{align*}
	Since Theorem~\ref{thm:continuity_incontext} furnishes the continuity of $\Psi^\star$ and since the mapping
	\begin{align*}
		\iota\colon E &\to C(D \times E; \R^r) \\
		y &\mapsto c_y
	\end{align*}
	is continuous, the lift
	\begin{equation*}
		\tilde \Psi^\star = \Psi^\star \circ (\mathrm{id}\times \iota)
	\end{equation*}
	is continuous as a composition of continuous maps, where we have introduced the map 
    \begin{equation*}
		\mathrm{id}\times \iota\colon (\rho, y) \mapsto \rho \oplus c_y .
	\end{equation*}
	Now, define the lift $\tilde E$ of the compact set $E$ by
	\begin{equation*}
		\tilde E \defeq \iota(E) = \set{c_y}{y \in E} \subset C(D \times E; \R^r) .
	\end{equation*}
	Since $\iota$ is continuous and $E$ is compact, so is $\tilde E$ and hence the Cartesian product
	\begin{equation*}
		K \times \tilde E \subset C(D \times E; \R^{r+1})
	\end{equation*}
	is also compact; see \cref{footnote:concat} for a note on the implicit function space identification occurring here.
Following the proof of Theorem~\ref{thm:approx_kernel}, we can first restrict $\tilde \Psi^\star$ to the compact set $K \times \tilde E$ and then apply the Dugundji extension theorem to obtain a continuous extension
	\begin{equation*}
		\tilde \Psi^\star_0\colon C(D\times E; \R^{r+1}) \to C(D)
	\end{equation*}
	satisfying
	\begin{equation}
		\label{eq:ext-matches-on-restriction}
		\tilde \Psi^\star_0 \big|_{K \times \tilde E} = \tilde \Psi^\star \big|_{K \times \tilde E} .
	\end{equation}
	We can finally apply a modification to~\cite[Theorem 1, p.~267]{lanthaler2025nonlocality} for distinct input and output domains, discussed in~\cite[p.~270]{lanthaler2025nonlocality}, to obtain a neural operator $\tilde \Psi_0$ such that
	\begin{equation*}
		\sup_{(\rho,c_y)\in K\times \tilde E}\norm{\tilde \Psi^\star_0(\rho, c_y)-\tilde \Psi_0(\rho, c_y)}_{C(D)}\leq \ep .
	\end{equation*}
	Finally, consider the vector-to-function layer $V$ given by
	\begin{align*}
		V\colon C(D\times E)\times E &\to C(D\times E ; \R^{r+1}) \\
		(\rho,y) &\mapsto (\rho, c_y)
	\end{align*}
	and note that it satisfies \cref{def:vector-to-function} with $h \equiv 1$ and $\psi = \textup{id}$ the constant and identity maps, respectively.  
	Now, define the augmented neural operator $\tilde\Psi$ by
	\begin{equation*}
		\widetilde{\Psi} = \tilde \Psi_0 \circ V .
	\end{equation*}
	By \eqref{eq:ext-matches-on-restriction} and the construction of $\tilde E$, $\tilde \Psi^\star$, and $V$, we have
	\begin{align*}
		\sup_{(\rho,y)\in K\times E}\norm{ \Psi^\star(\rho, y) - \tilde \Psi(\rho, y)}_{C(D)} 
		&= \sup_{(\rho,y)\in K\times E}\norm{ \Psi^\star(\rho, y) - (\tilde \Psi_0 \circ L)(\rho, y)}_{C(D)} \\
		&= \sup_{(\rho,c_y)\in K\times \tilde E}\norm{\tilde \Psi^\star(\rho, c_y)-\tilde \Psi_0(\rho, c_y)}_{C(D)} \\
		& = \sup_{(\rho,c_y)\in K\times \tilde E}\norm{\tilde \Psi^\star_0(\rho, c_y)-\tilde \Psi_0(\rho, c_y)}_{C(D)} \\
		&\leq \ep .
	\end{align*}
	The proof is complete.
\end{proof}

\corHolderMarkov*
\begin{proof}
	We show that $ K\defeq \cX_{\delta}\cap \cH^\alpha_R\subset C(D\times E) $ is compact and then apply \cref{thm:approx_kernel,thm:approx_incontext}.
	By \cref{lem:completeness-of-metric-spaces}, $ \cX_{\delta} $ is closed in $ C(D\times E) $. An argument similar to the one used to prove \cite[Lem.~4.4, pp.~25--26]{de2025extension} shows that the set $ \cH^\alpha_R $ is also closed in $ C(D\times E) $. Thus, $ \cX_{\delta}\cap \cH^\alpha_R $ is closed. It remains to show that $ \cH^\alpha_R\supseteq \cX_{\delta}\cap \cH^\alpha_R $ is relatively compact in $ C(D\times E) $. But this is a consequence of the compactness of the embedding $ C^{0,\al}(D\times E)\hookrightarrow C^{0,\beta}(D\times E) $ for any $ 0<\beta<\al $ \cite[Thm. 1.34, pp. 11--12]{adams2003sobolev}.
\end{proof}

\subsection{Proofs for Subsection~\ref{sec:unbounded}:~\nameref*{sec:unbounded}}
\thmUaMarkovUnbounded*
\begin{proof}
The proof is divided into several steps.
\begin{enumerate}[
    label=\textbf{Step~\arabic*.},
    leftmargin=*,
    labelsep=0.75em,
    topsep=1.67ex,
    itemsep=0.33ex,
    partopsep=1ex,
    parsep=1ex
]
    \item \textit{Approximate $K$ by another compact set $K_M$ that is bounded in $L^\infty$}.\label{item:step1}
    \item \textit{Truncate the output space $\R^{d'}$ to a bounded domain $\Omega'$}.\label{item:step2}
    \item \textit{Truncate the input space $\R^{d}$ to a bounded domain $\Omega$}.\label{item:step3}
    \item \textit{Apply an existing universal approximation theorem for operators between bounded domain function spaces}.\label{item:step4}
    \item \textit{Combine the estimates into a single error bound}.\label{item:step5}
\end{enumerate}

We now proceed with this multi-step proof procedure.
\paragraph*{\ref{item:step1}}
    By \cref{lem:map_continuous_compact_neighborhood,lem:trunc_bounds}, we may choose $M=M(\ep)>0$ such that
    \begin{align}\label{eqn:proof_temp_trunc_truth}
        \sup_{\rho\in K}\norm{F^\star(\rho)-F^\star(T_M(\rho))}_{L^1(\R^{d'})}\leq\frac{\ep}{2},
    \end{align}
    where $T_M$ is as in \eqref{eqn:trunc_defn}.
    Let $K_M\defeq T_M(K)\subset L^1(\R^d)$, which is compact because $T_M\colon L^1(\R^d)\to L^1(\R^d)$ is continuous by \cref{lem:trunc_bounds}. Moreover, $K_M$ is bounded in $L^\infty(\R^d)$ by \eqref{eqn:trunc_norm}.

\paragraph*{\ref{item:step2}}
    Since $F^\star$ is continuous, the set $K_M^\star\defeq F^\star(K_M)\subset L^1(\R^{d'})$ is compact. By a corollary to the Frech\'et--Kolmogorov--Riesz (FKR) compactness theorem \cite[Exercise~4.34(3), p.~128--129]{brezis2011functional}, $K_M^\star$ is uniformly tight. Thus, we can find a compact set $A'=A'(\ep)\subset \R^{d'}$ such that
    \begin{align}\label{eqn:proof_temp_tight_truth}
        \sup_{g\in K_M^\star}\int_{\R^{d'}\setminus A'}\abs{g(x)}\dd{x}\leq\frac{\ep}{4}.
    \end{align}
    By the compactness of $A'$, we may choose $R'=R'(\ep)$ large enough such that $\Omega'\defeq (-R',R')^{d'}\supset A'$. Note that $\Omega'$ is an open bounded Lipschitz domain.

\paragraph*{\ref{item:step3}}
    By the same FKR compactness argument, for any $\delta>0$, we can find a compact set $A=A(\delta,\ep)\subset\R^d$ such that
    \begin{align}\label{eqn:temp_proof_tight_set}
        \sup_{h \in K_M}\int_{\R^{d}\setminus A}\abs{h(x)}\dd{x}\leq\delta .
    \end{align}
    By the compactness of $A$, we may choose $R=R(\delta, \ep)$ large enough such that $\Omega\defeq (-R,R)^{d}\supset A$. Notice that $\norm{h-\cZ_\Omega\cR_\Omega h}_{L^1(\R^d)}=\int_{\R^{d}\setminus \Omega}\abs{h(x)}\dd{x}\leq \int_{\R^{d}\setminus A}\abs{h(x)}\dd{x}$. Then \cref{lem:map_continuous_compact_neighborhood} implies that we can choose $\delta=\delta(\ep)$ in \eqref{eqn:temp_proof_tight_set} such that for every $h\in K_M$, it holds that
    \begin{align}\label{eqn:temp_proof_input_trunc_domain}
        \norm{\cR_{\Omega'}F^\star(h)-\cR_{\Omega'}F^\star(\cZ_\Omega\cR_\Omega h)}_{L^1(\Omega')}\leq \norm{F^\star(h)-F^\star(\cZ_\Omega\cR_\Omega h)}_{L^1(\R^{d'})}\leq \frac{\ep}{8}.
    \end{align}

\paragraph*{\ref{item:step4}}
    Now define $\Psi_{0}\defeq \cR_{\Omega'}\circ F^\star \circ \cZ_{\Omega}\colon L^1(\Omega)\to L^1(\Omega')$, which is a composition of continuous maps and is hence continuous itself. The set $\cR_\Omega (K_M)\subset L^1(\Omega)$ is compact and is also bounded in $L^\infty(\Omega)$. By the universal approximation theorem for neural operators \cite[Thm.~2, Sec.~2.3, p.~267, applied with $s=s'=0$, $p=p'=1$, $k=k'=1$]{lanthaler2025nonlocality}, there exists a neural operator $\Psi_{00}\colon L^1(\Omega)\to L^1(\Omega')$ such that
    \begin{align}\label{eqn:temp_proof_ua_bounded}
        \sup_{f\in \cR_\Omega (K_M)}\norm{\Psi_{0}(f)-\Psi_{00}(f)}_{L^1(\Omega')}\leq \frac{\ep}{8}.
    \end{align}

\paragraph*{\ref{item:step5}}
    Define $\Psi\defeq \cZ_{\Omega'}\circ\Psi_{00}\circ\cR_\Omega\circ T_M$. Recall that $T_M(f)(x)=t_M(f(x))$, where $t_M(z)\defeq \min(z, M\, \mathrm{sgn}(z))$ because $f$ is real-valued. A direct calculation shows that
    \begin{align*}
       z\mapsto t_M(z)=\max(-M,\min(z,M))=(z+M)_+-(z-M)_+-M
    \end{align*}
    is a shallow ReLU neural network. We claim that $\Psi$ has the form \eqref{eqn:defn_neural_op_unbounded} with operator $\cS$ associated to function $S\defeq t_M$. Indeed, for any $x\in\Omega$, it holds that $(\cR_\Omega\circ T_M)(f)(x)=t_M(f(x))=t_M(R_\Omega f(x))=(\cS\circ \cR_\Omega)(f)(x)$. The claim follows.

    To complete the proof, we estimate the total error by
    \begin{align*}
		\sup_{\rho\in K}\norm{F^\star(\rho)-\Psi(\rho)}_{L^1(\R^{d'})}&\leq
        \sup_{\rho\in K}\norm{F^\star(\rho)-F^\star(T_M(\rho))}_{L^1(\R^{d'})} \\
        &\qquad + \sup_{h\in K_M}\norm{F^\star(h)-(\cZ_{\Omega'}\circ\Psi_{00}\circ\cR_\Omega)(h)}_{L^1(\R^{d'})}\\
        &\leq  \frac{\ep}{2} + \sup_{h\in K_M} \int_{\R^{d'}\setminus \Omega'}\abs{F^\star(h)(x)}\dd{x}\\
        &\qquad\qquad + \sup_{h\in K_M} \int_{\Omega'}\abs{F^\star(h)(x)-\Psi_{00}(\cR_\Omega h)(x)}\dd{x}\\
        &\leq \frac{3\ep}{4} + \sup_{h\in K_M} \int_{\Omega'}\abs{F^\star(h)(x)-\Psi_{00}(\cR_\Omega h)(x)}\dd{x}.
    \end{align*}
    We invoked \cref{eqn:proof_temp_tight_truth,eqn:proof_temp_trunc_truth} in the preceding display. Moreover,
    \begin{align*}
        \sup_{h\in K_M} \int_{\Omega'}\abs{F^\star(h)(x)-\Psi_{00}(\cR_\Omega h)(x)}\dd{x}
        &\leq \sup_{h\in K_M}\norm{\cR_{\Omega'}F^\star(h)-\Psi_{0}(\cR_\Omega h)}_{L^1(\Omega')}\\
        &\qquad + \sup_{h\in K_M}\norm{\Psi_{0}(\cR_\Omega h)-\Psi_{00}(\cR_\Omega h)}_{L^1(\Omega')}\\
        &\leq \frac{\ep}{8} + \frac{\ep}{8}=\frac{\ep}{4}
    \end{align*}
    by \cref{eqn:temp_proof_ua_bounded,eqn:temp_proof_input_trunc_domain}.
    Combining the estimates concludes the proof.
\end{proof}

\corUaMarkovUnboundedConditioning*
\begin{proof}
    Define $\cZ_B^{(2)}\in\cL(L^1(\R^m\times B); L^1(\R^m\times\R^r))$ by zero extension in the second coordinate. Also define $F^\star\defeq \cZ_B^{(2)}\circ \cG^\star|_K\colon K\subset L^1(\R^m\times\R^r)\to L^1(\R^m\times\R^r)$, which is continuous by \cref{lem:global_lip_kernel_operator_unbounded}.
	By the Dugundji extension theorem \cite[Thm.~4.1, p.~357]{dugundji1951extension}, there exists a continuous map $ \widetilde{F}^\star\colon L^1(\R^m\times\R^r)\to L^1(\R^m\times\R^r) $ that extends $ F^\star $. With $ \widetilde{F}^\star\colon L^1(\R^{m+r})\to L^1(\R^{m+r}) $, application of \cref{thm:ua_markov_unbounded} (taking the larger of the two radii) delivers $R>0$ and a whole space neural operator $\cG\colon L^1(\R^{m+r})\to L^1(\R^{m+r})$ as in \eqref{eqn:defn_neural_op_unbounded} with $\Omega=\Omega'\defeq (-R,R)^{m+r}$ such that
    \begin{align*}
        \sup_{\rho\in K}\norm{\widetilde{F}^\star(\rho)-\cG(\rho)}_{L^1(\R^{m+r})}\leq\ep.
    \end{align*}
    The assertion \eqref{eqn:ua_markov_unbounded_conditioning} follows from the fact that
    \begin{align*}
        \norm{\cG^\star(\rho)-\cG(\rho)}_{L^1(\R^{m}\times B)}\leq
        \norm{{F}^\star(\rho)-\cG(\rho)}_{L^1(\R^{m}\times\R^r)}=\norm{\widetilde{F}^\star(\rho)-\cG(\rho)}_{L^1(\R^{m+r})}
    \end{align*}
    for any $\rho\in K$.
\end{proof}

\begin{restatable}[Truncation operator]{lemma}{lemTruncBounds}\label{lem:trunc_bounds}
	Let $U\subseteq\R^d$ be measurable and open. Fix $M>0$. Define the truncation operator $T_M\colon L^1(U)\to L^1(U)\cap L^\infty(U)\subseteq L^1(U)$ by
	\begin{align}\label{eqn:trunc_defn}
		T_M(f)(x)\defeq
		\begin{cases}
			f(x), &\quad \abs{f(x)}\leq M,\\
			M\frac{f(x)}{\abs{f(x)}},&\quad \abs{f(x)}>M
		\end{cases}
	\end{align}
	for each $f\in L^1(U)$ and $x\in U$. Then $T_M$ is continuous. Moreover, if $K\subset L^1(U)$ is compact, then
	\begin{subequations}\label{eqn:trunc_both}
		\begin{align}
			\sup_{f\in K}\norm{T_M(f)}_{L^\infty(U)}&\leq M\qa \label{eqn:trunc_norm}\\
			\lim_{M\to\infty}\sup_{f\in K}\norm{f-T_M(f)}_{L^1(U)}&=0.\label{eqn:trunc_converge}
		\end{align}
	\end{subequations}
\end{restatable}
\begin{proof}
	The continuity of $T_M\colon L^1(U)\to  L^1(U)$ follows from the fact that $T_M$ is a Nemytskii operator associated to a $1$-Lipschitz function from $U$ into itself. The uniform bound \eqref{eqn:trunc_norm} follows by the definition \eqref{eqn:trunc_defn} of $T_M$. For the assertion \eqref{eqn:trunc_converge}, we compute
	\begin{align*}
		\int_{U}\abs{f(x)-T_M(f)(x)}\dd{x}&=\int_{U}\one_{\{\abs{f(x)}>M\}}\abs[\bigg]{f(x)\biggl(1-\frac{M}{\abs{f(x)}}\biggr)}\dd{x}\\
		&\leq \int_{U}\one_{\{\abs{f(x)}>M\}}\abs{f(x)}\dd{x}
	\end{align*}
	for any $f\in K$ and $M>0$. 
    Using the fact that $K\subset L^1(U)$ compact implies that $K$ is uniformly integrable \cite[Exercise~4.36(1--2), Eqn.~(d), p.~129]{brezis2011functional}, the result \eqref{eqn:trunc_converge} follows.
\end{proof}

\begin{restatable}[Uniform continuity near compacta]{lemma}{lemMapContinuousCompactNeighborhood}\label{lem:map_continuous_compact_neighborhood}
	Let $(X, \sfd_X)$ and $ (Y,\sfd_Y) $ be metric spaces. Let $F\colon X\to Y$ be continuous. Let $K\subset X$ be compact. Then for every $\ep>0$, there exists $\delta>0$ such that for all $k\in K$ and all $x\in X$ satisfying $\sfd_X(k,x)\leq \delta$, it holds that $ \sfd_Y(F(k),F(x))\leq\ep $.
\end{restatable}
\begin{proof}
	Suppose not. Then there exists $\ep>0$ such that for every $n\in\N$, there exists $(k_n,x_n)\in K\times X$ such that $\sfd_X(k_n,x_n)\leq 1/n$ and $\sfd_Y(F(k_n),F(x_n))>\ep$. By the compactness of $K$ in $X$, there exists a subsequence $\{k_{n_j}\}_{j\in\N}$ and a limit $k\in K$ such that $\sfd_X(k_{n_j},k)\to 0$ as $j\to\infty$. Consequently, $\sfd_X(x_{n_j}, k)\leq 1/n_j +\sfd_X(k_{n_j}, k)\to 0$. By continuity of $F$, we see that $\sfd_Y(F(k_{n_j}), F(x_{n_j}))\leq \sfd_Y(F(k_{n_j}), F(k)) + \sfd_Y(F(k), F(x_{n_j}))\to 0$ as $j\to\infty$. This contradicts the fact that $\sfd_Y(F(k_n),F(x_n))>\ep$ for all $n\in\N$.
\end{proof}

\section{Proofs for Section~\ref{sec:stability}:~\nameref*{sec:stability}}\label{app-subsec:proofs_stability}
This appendix collects the proofs of the stability results for the conditioning operators.

\subsection{Proofs for Subsection~\ref{subsec:stability-results}:~\nameref*{subsec:stability-results}}

\lemLocalLipKernelOperator*
\begin{proof}
	By definition of $ \cG^\star $ and the triangle inequality,
	\begin{align*}
		\norm{\cG^\star(p)-\cG^\star(q)}_{C(D\times E)}=\norm[\bigg]{\frac{p}{m_p}-\frac{q}{m_q}}_{C(D\times E)}\leq \norm[\bigg]{\frac{p}{m_p}-\frac{p}{m_q}}_{C(D\times E)} + \norm[\bigg]{\frac{p}{m_q}-\frac{q}{m_q}}_{C(D\times E)} .
	\end{align*}
	Since both $ p $ and $ q $ belong to $ \cX_{\delta}\subset \cM_\delta $, the second term on the right-hand side of the previous display is bounded above by $ \delta^{-1}\norm{p-q}_{C(D\times E)} $.
	The first term is bounded above by
	\begin{align*}
		\norm{p}_{C(D\times E)}\norm[\bigg]{\frac{m_q-m_p}{m_pm_q}}_{C(D\times E)}\leq \frac{\norm{p}_{C(D\times E)}}{\delta^2}\norm{m_p-m_q}_{C(E)}.
	\end{align*}
	The data marginals satisfy
	\begin{align*}
		\norm{m_p-m_q}_{C(E)}=\sup_{y\in E}\abs[\bigg]{\int_D \bigl(p(x,y)-q(x,y)\bigr)\dd{x}}\leq \abs{D}\, \norm{p-q}_{C(D\times E)} .
	\end{align*}
	The asserted stability bound follows by collecting terms and invoking symmetry.
\end{proof}

\lemGlobalLipKernelOperatorUnbounded*
\begin{proof}
	By definition of $ \cG^\star $ and the triangle inequality,
	\begin{align*}
		\norm{\cG^\star(p)-\cG^\star(q)}_{L^1(\R^m\times B)}\leq \norm[\bigg]{\frac{p}{m_p}-\frac{p}{m_q}}_{L^1(\R^m\times B)} + \norm[\bigg]{\frac{p}{m_q}-\frac{q}{m_q}}_{L^1(\R^m\times B)}.
	\end{align*}
	Since both $ p $ and $ q $ belong to $ \cU_{\delta}$, the second term on the right-hand side of the previous display is bounded above by $ \delta^{-1}\norm{p-q}_{L^1(\R^m\times\R^r)} $.
	For the first term,
	\begin{align*}
		\norm[\bigg]{\frac{p}{m_p}-\frac{p}{m_q}}_{L^1(\R^m\times B)}&=\int_{\R^m\times B}\frac{\abs{m_q(y)-m_p(y)}}{m_p(y)m_q(y)} p(x,y)\dd{x}dy\\
        &=\int_B \frac{\abs{m_q(y)-m_p(y)}}{m_q(y)} \dd{y}\\
        &\leq \frac{1}{\delta}\norm{m_p-m_q}_{L^1(B)}.
	\end{align*}
	The second equality in the preceding display invoked Tonelli's theorem to exchange the order of integration. The data marginals satisfy
	\begin{align*}
		\norm{m_p-m_q}_{L^1(B)}=\int_B\abs[\bigg]{\int_{\R^m} \bigl(p(x,y)-q(x,y)\bigr)\dd{x}}\dd{y}\leq \norm{p-q}_{L^1(\R^m\times\R^r)}.
	\end{align*}
	The asserted Lipschitz continuity follows.
\end{proof}

\thmContinuityIncontext*
\begin{proof}
	We use the sequential definition of continuity. Let $ (\rho,y)\in\cX_{\delta}\times E $ and $ \{(\rho_n,y_n)\}_{n\in\N}\subset \cX_{\delta}\times E $ be any sequence such that
	\begin{align*}
		\norm{(\rho_n,y_n)-(\rho,y)}_{C(D\times E)\times E}=\max\bigl(\norm{\rho_n-\rho}_{C(D\times E)}, \abs{y_n-y}\bigr)\to 0
	\end{align*}
	as $ n\to\infty$. We must prove that $ \Psi^\star(\rho_n,y_n)\to\Psi^\star(\rho,y) $ in $ C(D) $. To this end, we use the error decomposition
	\begin{align*}
		\norm{\Psi^\star(\rho_n,y_n)-\Psi^\star(\rho,y)}_{C(D)} &\leq \underbrace{\norm{\rho_n(\slot\condbar y_n)-\rho(\slot\condbar y_n)}_{C(D)}}_{\mathrm{(I)}}\\[1mm]
		&\qquad\qquad +
		\underbrace{\norm{\rho(\slot\condbar y_n)-\rho(\slot\condbar y)}_{C(D)}}_{\mathrm{(II)}}.
	\end{align*}

	For the first term $ \mathrm{(I)} $, it holds that
	\begin{align*}
		\mathrm{(I)}&\leq \sup_{y'\in E}\norm{\rho_n(\slot\condbar y')-\rho(\slot\condbar y')}_{C(D)}\\
		&=\norm{\cG^\star(\rho_n)-\cG^\star(\rho)}_{C(D\times E)}\\
		&\leq
		\frac{1}{\delta} \biggl(1+\frac{\abs{D}\min(\sup_{n'\in\N}\norm{\rho_{n'}}_{C(D\times E)},\norm{\rho}_{C(D\times E)})}{\delta}\biggr)\norm{\rho_n-\rho}_{C(D\times E)}
	\end{align*}
	by \cref{lem:local_lip_kernel_operator}. The last line of the preceding display tends to zero as $ n\to\infty $.

	It remains to show that term $ \mathrm{(II)} $ converges to zero. It holds that
	\begin{align*}
		\mathrm{(II)}&\leq \norm[\bigg]{\frac{\rho(\slot, y_n)}{m_\rho(y_n)}-\frac{\rho(\slot, y_n)}{m_\rho(y)}}_{C(D)} + \norm[\bigg]{\frac{\rho(\slot, y_n)}{m_\rho(y)}-\frac{\rho(\slot, y)}{m_\rho(y)}}_{C(D)}\\
		&\leq
		\frac{\norm{\rho}_{C(D\times E)}\abs{m_\rho(y_n)-m_\rho(y)}}{\delta^2} + \frac{1}{\delta}\norm{\rho(\slot, y_n)-\rho(\slot, y)}_{C(D)}.
	\end{align*}
	Since $ \rho\in C(D\times E) $ is continuous on the compact set $ D\times E $, the function $ \rho $ is in fact uniformly continuous with modulus of continuity $ \omega_\rho $. Therefore,
	\begin{align*}
		\abs{m_\rho(y_n)-m_\rho(y)}\leq \int_D\abs{\rho(x,y_n)-\rho(x,y)}\dd{x} \leq \abs{D}\, \omega_\rho\bigl(\max(0,\abs{y_n-y})\bigr)\to 0
	\end{align*}
	as $ n\to\infty $. Similarly,
	\begin{align*}
		\norm{\rho(\slot, y_n)-\rho(\slot, y)}_{C(D)}=\sup_{x\in D} \abs{\rho(x,y_n)-\rho(x,y)}\leq \omega_\rho(\abs{y_n-y})\to 0
	\end{align*}
	by the convergence of $ y_n\to y $. We deduce that $ \mathrm{(II)}\to 0 $ as required.
\end{proof}

\corHolderStabilityIncontext*
\begin{proof}
	First suppose that $ \norm{p-q}_{C(D\times E)}\leq 1 $. By following the proof of \cref{thm:continuity_incontext}, we have
	\begin{align*}
		\norm{\Psi^\star(p,y)-\Psi^\star(q,y')}_{C(D)}&\leq \frac{1}{\delta} \biggl(1+\frac{\abs{D}R}{\delta}\biggr)\norm{p-q}_{C(D\times E)}\\
		&\qquad + \frac{\abs{D}R}{\delta^2}\omega_p(\abs{y-y'})
		+
		\frac{1}{\delta}\omega_p(\abs{y-y'})\\
		&\leq \frac{1}{\delta} \biggl(1+\frac{\abs{D}R}{\delta}\biggr)\norm{p-q}_{C(D\times E)}\\
		&\qquad + \frac{\abs{D}R^2}{\delta^2}\abs{y-y'}^\al
		+
		\frac{R}{\delta}\abs{y-y'}^\al.
	\end{align*}
	Since $ \norm{p-q}_{C(D\times E)}\leq 1 $, it holds that
	\begin{align*}
		\norm{p-q}_{C(D\times E)}\leq \norm{p-q}_{C(D\times E)}^\al\leq \norm{(p,y)-(q,y')}^\al_{C(D\times E)\times E}\,.
	\end{align*}
	Using $\abs{y-y'}\leq \norm{(p,y)-(q,y')}_{C(D\times E)\times E}$ and combining terms yield the bound
	\begin{align*}
		\norm{\Psi^\star(p,y)-\Psi^\star(q,y')}_{C(D)}&\leq \frac{(1+R)}{\delta}\biggl(1+\frac{\abs{D}R}{\delta}\biggr)\norm[\big]{(p,y)-(q,y')}^\al_{C(D\times E)\times E}\\
		&\leq \frac{(1+2R)}{\delta}\biggl(1+\frac{\abs{D}R}{\delta}\biggr)\norm[\big]{(p,y)-(q,y')}^\al_{C(D\times E)\times E}\,.
	\end{align*}
	We enlarged $ R $ to $ 2R $ in the first parenthesis of the last line.

	If $ \norm{p-q}_{C(D\times E)}> 1 $, then $\norm{(p,y)-(q,y')}^\al_{C(D\times E)\times E}\geq  \norm{p-q}_{C(D\times E)}^\al > 1 $. Using this fact, it holds that
	\begin{align*}
		\norm{\Psi^\star(p,y)-\Psi^\star(q,y')}_{C(D)}&\leq \norm{\Psi^\star(p,y)}_{C(D)} + \norm{\Psi^\star(q,y')}_{C(D)}\\
		&\leq \frac{2R}{\delta}\\
		&< \frac{2R}{\delta} \norm{(p,y)-(q,y')}^\al_{C(D\times E)\times E}\\
		&\leq \frac{(1+2R)}{\delta}\biggl(1+\frac{\abs{D}R}{\delta}\biggr)\norm[\big]{(p,y)-(q,y')}^\al_{C(D\times E)\times E}\,.
	\end{align*}
	Thus, the estimate \eqref{eqn:holder_stability_incontext} is always valid as asserted.
\end{proof}

\subsection{Proofs for Subsection~\ref{sec:removing_unif_lower_bound}:~\nameref*{sec:removing_unif_lower_bound}}\label{app-subsec:proofs_generalizations}

\begin{restatable}[Marginal continuity]{lemma}{lemContinuityMarginalFunction}
	\label{lem:continuity_marginal_function}
	The function $ m_\rho \colon E \to \R $ is continuous for any $ \rho \in \Cac(D \times E) $.
\end{restatable}
\begin{proof}
	Since the function $ \rho \colon D \times E \to \R $ is continuous on the compact set $ D \times E $, it is in fact uniformly continuous.
	Take a sequence $ \{y_n\}_{n \in \N} \subset E $ such that $ y_n \to y $ as $ n \to \infty $. By uniform continuity, there is some $n_0 \in \N$ such that for all $ n \geq n_0 $, we have
	\begin{equation*}
		\sup_{x \in D} |\rho(x,y_n) - \rho(x,y)| < \epsilon / |D| ,
	\end{equation*}
	where $|D|$ is the Lebesgue measure of $ D $.
	Then for all $ n \geq n_0 $, we have
	\begin{align*}
		|m_\rho(y_n) - m_\rho(y)| & = \left| \int_D \rho(x,y_n) \, dx - \int_D \rho(x,y) \, dx \right| \\
		& \leq \int_D |\rho(x,y_n) - \rho(x,y)| \, dx \\
		& \leq |D| \, \sup_{x \in D} |\rho(x,y_n) - \rho(x,y)| \\
		& < \epsilon ,
	\end{align*}
	concluding the proof.
\end{proof}

\begin{restatable}[Positive marginal characterization]{lemma}{lemmaRepresentationOfXPlus}
	\label{lemma:representation_of_X_plus}
	Let $\cX_+\defeq \bigcup_{\delta > 0} \mathcal{X}_\delta$. Then it holds that
    \begin{equation*}
		\mathcal{X}_+ = \set{\rho \in \Cac(D \times E)}{m_\rho > 0} .
	\end{equation*}
\end{restatable}
\begin{proof}
	Since for any $\delta > 0$ we have
	\begin{equation*}
		\mathcal{X}_\delta \subseteq \set{\rho \in \Cac(D \times E)}{m_\rho > 0} ,
	\end{equation*}
	we get the trivial inclusion
	\begin{equation*}
		\mathcal{X}_+ \subseteq \set{\rho \in \Cac(D \times E)}{m_\rho > 0} .
	\end{equation*}
	For the other inclusion, take any $\rho \in C(D \times E)$ such that $m_\rho > 0$. By compactness of $E$ and \cref{lem:continuity_marginal_function},
	there is some $y_0 \in E$ such that
	\begin{equation*}
		m_\rho(y_0) = \min_{y' \in E} m_\rho(y') .
	\end{equation*}
	Thus, if we set $\delta \defeq m_\rho(y_0) > 0$, then we obtain $\rho \in \mathcal{X}_\delta \subseteq \mathcal{X}_+$ as required.
\end{proof}

\propExtensionOfOperators*
\begin{proof}
	For any $\delta > 0$, \cref{eqn:operator_kernel,eqn:operator_incontext} define operators
	\begin{equation*}
		\cG^\star_\delta \colon \mathcal{X}_\delta \to C(D \times E) \qa \Psi^\star_\delta \colon \mathcal{X}_\delta \times E \to C(D) .
	\end{equation*}
	Here, we are merely making the dependence on $\delta$ explicit in the notation. By \cref{lem:local_lip_kernel_operator,thm:continuity_incontext}, these operators are continuous.
	We can now define the extensions
	\begin{align*}
		& \widetilde{\cG}^\star \colon \mathcal{X}_+ \to C(D \times E) \qa \\
		& \widetilde{\Psi}^\star \colon \mathcal{X}_+ \times E \to C(D)
	\end{align*}
	by setting
	\begin{equation*}
		\widetilde{\cG}^\star|_{\mathcal{X}_\delta} = \cG^\star_\delta \qa \widetilde{\Psi}^\star|_{\mathcal{X}_\delta \times E} = \Psi^\star_\delta .
	\end{equation*}
	These define valid extensions since \cref{lemma:representation_of_X_plus} implies that for any $\rho \in \mathcal{X}_+$, there is some $\delta > 0$ such that $\rho \in \mathcal{X}_\delta$. Therefore,
	\begin{equation*}
		\widetilde{\cG}^\star(\rho) = \cG^\star_\delta(\rho) \qa \widetilde{\Psi}^\star(\rho, y) = \Psi^\star_\delta(\rho, y) ,
	\end{equation*}
	and the operators on the right-hand side are defined in the preceding discussion.
	Moreover, note that these expressions are well-defined because for any two $\delta_1>0$ and $\delta_2 > 0$ such that $\delta_1 > \delta_2 > 0$,\cref{eqn:operator_kernel,eqn:operator_incontext} yield
	\begin{equation*}
		\cG^\star_{\delta_1} = \cG^\star_{\delta_2} \qa\Psi^\star_{\delta_1} = \Psi^\star_{\delta_2} .
	\end{equation*}
	It remains to show that these extensions are continuous and unique.
	Consider a sequence $\{\rho_n\}_{n \in \N} \subset \mathcal{X}_+$ such that
	\begin{equation*}
		\rho_n \xrightarrow{L^\infty} \rho
	\end{equation*}
	and $\rho \in \mathcal{X}_+$.
	By \cref{lemma:representation_of_X_plus}, there exists some $\delta > 0$ such that $\rho \in \mathcal{X}_\delta$, and therefore by definition
	\begin{equation*}
		m_\rho \geq \delta > 0 .
	\end{equation*}
	Now by convergence in $L^\infty$, there must exist $n_0 \in \N$ such that
	\begin{equation*}
		n \geq n_0 \implies \| \rho_n - \rho_{n_0} \|_{L^\infty} \leq \frac{1}{2} \| \rho_{n_0} \|_{L^\infty} ,
	\end{equation*}
	and therefore for $(x,y) \in D \times E$ and $n \geq n_0$, one has
	\begin{equation*}
		\rho_n(x,y) \geq \Big| \rho_{n_0}(x,y) - \big|\rho_n(x,y) - \rho_{n_0}(x,y) \big| \Big| \geq \frac{1}{2} \rho_{n_0}(x,y) .
	\end{equation*}
	By taking infima, we have for all $(x,y) \in D \times E$ and $n \geq n_0$ that
	\begin{equation*}
		\rho_{n_0}(x,y) \geq \frac{1}{2} \inf_{(x,y) \in D \times E} \rho_{n_0}(x,y) \geq \frac{\delta}{2} > 0 .
	\end{equation*}
	To summarize, for all $n \geq n_0$, we have
	\begin{equation*}
		\rho_n \geq \frac{\delta}{2} .
	\end{equation*}
	Integrating both sides over $D$ we get
	\begin{equation*}
		m_{\rho_n}(y) \geq \frac{\delta}{2} \, |D| \eqdef \delta' > 0 ,
	\end{equation*}
	and so it follows that for any $n \geq n_0$, we have
	\begin{equation*}
		\rho_n \in \mathcal{X}_{\delta'} .
	\end{equation*}
	Now since
	\begin{equation*}
		\widetilde{\cG}^\star|_{\mathcal{X}_{\delta'}} = \cG^\star_{\delta'}
        \qa
        \widetilde{\Psi}^\star|_{\mathcal{X}_{\delta'} \times E} = \Psi^\star_{\delta'} ,
	\end{equation*}
	continuity follows by the continuity of $\cG^\star_{\delta'}$ and $\Psi^\star_{\delta'}$.
	For uniqueness, pick any $\rho \in \mathcal{X}_+$ and note that since $\rho$ is continuous and $D \times E$, there exists some $(x_0, y_0) \in D \times E$ such that
	\begin{equation*}
		\inf_{(x,y) \in D \times E} \rho(x,y) = \rho(x_0, y_0) > 0 .
	\end{equation*}
	Letting $\delta'' \defeq \rho(x_0, y_0)$, we note that any extensions $\cG_0^\star$ and $\Psi_0^\star$ must agree with $\cG^\star_{\delta''}$ and $\Psi^\star_{\delta''}$ on $\mathcal{X}_{\delta''}$, which contains $\rho$ by definition of $\delta''$. Thus, we have
	\begin{equation*}
		\cG_0^\star(\rho) = \cG^\star_{\delta''}(\rho) = \tilde \cG(\rho) \qa \Psi_0^\star(\rho, y) = \Psi^\star_{\delta''}(\rho, y) = \tilde \Psi^\star(\rho, y) \qfa y \in E .
	\end{equation*}
	In other words, the arbitrary extensions $\cG_0^\star$ and $\Psi_0^\star$ agree with $\tilde \cG^\star$ and $\tilde \Psi^\star$ at $\rho$, proving uniqueness.
\end{proof}

\begin{restatable}[Density of $\cX_+$]{proposition}{propDensityOfX}
	\label{prop:density-of-X}
	The space $\mathcal{X}_+ $ is dense in $ \Cac(D \times E) $ in the $ L^\infty $ norm, i.e., the supremum norm.
\end{restatable}
\begin{proof}
	First, since $D \times E \subseteq \R^{m+r}$, we can extend any $\rho \in \Cac(D \times E)$ by zero to the whole of $\R^{m+r}$ obtaining a discontinuous but bounded function
	\begin{equation*}
		\tilde \rho \in L^\infty(\R^{m+r}) \setminus C(\R^{m+r}) .
	\end{equation*}
	Now consider the family of mollifiers
	\begin{equation*}
		\chi_\eps(x) \propto e^{- \frac{\|x\|^2}{2 \, \epsilon}} , \quad \epsilon > 0 ,
	\end{equation*}
	with normalization constant chosen so that $\int_{\R^{m+r}} \chi_\eps(x) \, dx = 1$. It is a standard fact that the family $\{ \chi_\eps \}_{\epsilon > 0}$ is \emph{an approximate identity} in the sense of~\cite[Definition~1.2.15]{grafakos2008classical}, where we have taken our locally compact group to be $\R^{m+r}$ with the usual addition operation and the Haar measure coincides with the Lebesgue measure.
	Since $\tilde \rho \in L^1(\R^{m+r})$, we can define the mollified functions
	\begin{equation*}
		\tilde \rho_\eps = \tilde \rho * \chi_\eps ,
	\end{equation*}
	and since $\chi_\eps \in C^\infty(\R^{m+r})$, we have that $\tilde \rho_\eps \in C^\infty(\R^{m+r})$. Now by compactness of $D \times E$, there is an $R>0$ such that
	\begin{equation*}
		D \times E \subseteq B_R(0) .
	\end{equation*}
	Thus, for all $z \in D \times E$, we have
	\begin{align*}
		\tilde \rho_\eps(z) &= \int_{\R^{m+r}} \tilde \rho(z - y) \, \chi_\eps(y) \, dy \\
		&= \int_{D \times E} \rho(w) \, \chi_\eps(z - w) \, dw \\
		& \geq \inf_{w \in B_{2R}(0)} \chi_\eps(z - w) \, \int_{D \times E} \rho(w) \, dw \\
		& = \inf_{w \in B_{2R}(0)} \chi_\eps(z - w) \eqdef \delta  > 0 ,
	\end{align*}
	where in the second line we have used the commutativity of the convolution together with the support of $\tilde \rho$ and on the third line we have used that for all $z$ and  $w$ belonging to $ D \times E$, it holds that $\abs{z - w} \leq 2R$ by the triangle inequality. Finally, in the last line we have used that $\chi_\eps$ is strictly positive on any compact set. It follows that
	\begin{equation*}
		\rho_\epsilon \defeq \frac{\tilde \rho_\eps}{\int_{D \times E} \tilde \rho_\eps} \in \mathcal{X}_+ .
	\end{equation*}
	Last, since $\rho \in C(D \times E)$ and $D \times E$ is compact, we have that $\rho$ is uniformly continuous on $D \times E$. Thus, by~\cite[Theorem~1.2.19]{grafakos2008classical}, we have that
	\begin{equation*}
		\lim_{\eps \to 0} \| \tilde \rho_\eps - \tilde \rho \|_{L^\infty(D \times E)} = 0 .
	\end{equation*}
	Moreover, the dominated convergence theorem implies that
	\begin{equation*}
		\lim_{\eps \to 0} \int_{D \times E} \tilde \rho_\eps = \int_{D \times E} \tilde \rho = 1 ,
	\end{equation*}
	and therefore
	\begin{equation*}
		\lim_{\eps \to 0} \| \rho_\eps - \rho \|_{L^\infty(D \times E)} = 0 .
	\end{equation*}
	We have thus constructed a sequence $\{ \rho_\eps \}_{\eps > 0} \subset \mathcal{X}_+$ such that $\rho_\epsilon \to \rho$ in $L^\infty(D \times E)$ as $\eps \to 0$, for any $\rho \in C_{\textup{ac}}(D \times E)$, thereby completing the proof.
\end{proof}

\thmExtensionOfPsi*
\begin{proof}
	Suppose there is a continuous extension $\widetilde{\Psi}^\star \colon S_1 \times S_2 \to C(D)$ of $\Psi^\star$ to $S_1 \times S_2$.
	By \cref{prop:density-of-X}, there is a sequence $\{ \rho_n \}_{\epsilon \in [0,1]} \subseteq \mathcal{X}_+$ such that $\rho_\epsilon \to \rho$ in $L^\infty(D \times E)$ as $ \epsilon \to 0 $.
	By the assumed continuity of $\widetilde{\Psi}^\star$, we must have
	\begin{equation*}
		\widetilde{\Psi}^\star(\rho_\epsilon,y) \xrightarrow{L^\infty} \widetilde{\Psi}^\star(\rho,y) \qas \epsilon \to 0 .
	\end{equation*}
	Since $\widetilde{\Psi}^\star$ must agree with $\Psi^\star$ on $\mathcal{X}_+$, pointwise in $x \in D$ we have
	\begin{equation}
		\label{eq:pointwise_convergence_of_extension}
		{\Psi}^\star(\rho_\epsilon,y)(x) \to \widetilde{\Psi}^\star(\rho,y)(x) \qas \epsilon \to 0 .
	\end{equation}
	Now, by the construction in the proof of \cref{prop:density-of-X}, we have that $\rho_\epsilon$ is precisely then re-normalized mollification of $\rho$ with the Gaussian kernel $\chi_\epsilon$, namely
	\begin{equation*}
		\rho_\epsilon = \frac{\rho * \chi_\epsilon}{\int_{D \times E} \rho * \chi_\epsilon} .
	\end{equation*}
	By definition of $\Psi^\star$ and the above preceding, we have that
	\begin{equation}
		\label{eq:explicit-representation_of_extension}
		\Psi^\star(\rho_\epsilon,y)(x) = \frac{\rho_\epsilon(x, y)}{m_{\rho_\epsilon}(y)} = \frac{\rho * \chi_\epsilon(x,y) \left( \int_{D \times E} \rho * \chi_\epsilon \right)^{-1}}{m_{\rho_\epsilon}(y) \left( \int_{D \times E} \rho * \chi_\epsilon \right)^{-1}} = \frac{\rho * \chi_\epsilon(x,y)}{m_{\rho_\epsilon}(y)} .
	\end{equation}
	Now observe that
	\begin{align*}
		m_{\rho \ast \chi_\epsilon}(y) &= \int_D \rho * \chi_\epsilon(x,y) \, dx \\
		&= \int_D \int_{D \times E} \rho(z, w) \, \chi_\epsilon(x - z, y - w) \, dz \, dw \, dx \\
		&= \int_{D \times E} \rho(z,w) \, \left(\int_D \chi_\epsilon(x - z, y - w) \, dx \right) \, dz \, dw \\
		&= \int_{D \times E} \rho(z,w) \, m_{\chi_\epsilon}(y - w) \, dz \, dw \\
		&= \int_{E} \left( \int_{D} \rho(z,w) \, dz \right) m_{\chi_\epsilon}(y - w) \, dw \\
		&= \int_E m_\rho(w) \, m_{\chi_\epsilon}(y - w) \, dw \\
		&= m_\rho * m_{\chi_\epsilon}(y) ,
	\end{align*}
	where we have applied the Fubini--Tonelli theorem repeatedly. Combining this with \cref{eq:pointwise_convergence_of_extension,eq:explicit-representation_of_extension}, the assertion follows.
\end{proof}

\begin{lemma}[Products]\label{lemma:extension-of-Psi-to-indep-densitites}
	Consider the extended domain $S_1 \times S_2$ given by taking $S_2 = E$ and 
	\begin{equation*}
		S_1 \defeq \mathcal{X}_+ \bigcup \set{\rho\colon (x,y) \mapsto f(x) \, g(y)}{ f \in C(D)\qa g \in C(E) } .
	\end{equation*}
    For all $(\rho, y) \in S_1 \times S_2$, the limit in \cref{thm:extension-of-Psi} exists and equals
	\begin{equation*}		
        \Psi^\star(\rho,y) = f .
	\end{equation*}
\end{lemma}
\begin{proof}
    Indeed, for $\chi_\epsilon$ a Gaussian density as in \cref{thm:extension-of-Psi} we can write
	$
	\chi_\epsilon(x,y) = \chi_\epsilon^D(x) \, \chi_\epsilon^E(y) ,
	$
	where $\chi_\epsilon^D$ and $\chi_\epsilon^E$ are the densities of mean-centered Gaussians with covariance $\epsilon I$ on $\R^m$ and $\R^r$, respectively.
	For each $x$, we now compute
	\begin{align*}
		\Psi^\star(\rho,y)(x) = \lim_{\epsilon \searrow 0} \frac{\rho * \chi_\epsilon(x,y)}{m_{\rho} * m_{\chi_\epsilon}(y)} &= \lim_{\epsilon \searrow 0}\frac{\int_{D \times E} f(z) \, g(w) \, \chi_\epsilon(x - z, y - w) \, dz \, dw}{\int_E g(w) \, m_{\chi_\epsilon}(y - w) \, dw} \\
		&= \lim_{\epsilon \searrow 0} \frac{\left(\int_{D} f(z) \, \chi_\epsilon^D(x - z) \, dz \right) \left(\int_E g(w) \, \chi_\epsilon^E(y - w) \, dw \right)}{\int_E g(w) \, \chi_\epsilon^E(y - w) \, dw} \\
		&= \lim_{\epsilon \searrow 0} \frac{\chi_\epsilon^D * f(x) \, \chi_\epsilon^E * g(y)}{\chi_\epsilon^E * g(y)} \\
		&= \lim_{\epsilon \searrow 0} \chi_\epsilon^D * f(x) \\
		&= f(x) .
	\end{align*}
    This completes the proof.
\end{proof}

\section{Examples and remarks}
This appendix expands on aspects of the main text.

\begin{remark}[Exponential map]\label{rk:exponential-map}
    We note that the in-context conditioning operator $ \Psi^\star $ can be expressed as a composition of the kernel conditioning operator $ \cG^\star $ and the evaluation operator $ \mathrm{ev}\colon C(D\times E)\times E\to C(D) $ defined by $ \mathrm{ev}(f,y)=f(\slot,y) $. In particular, it holds that
    \begin{align*}
        & \Psi^\star\colon \domain(\Psi^\star) \xrightarrow{\cG^\star \times \textup{id}} C(D\times E) \times E \xrightarrow{\mathrm{ev}(\slot,y)} C(D) \\
        & \Psi^\star\colon (\rho,y) \mapsto (\kappa_\rho,y) \mapsto \kappa_\rho(\slot,y)=\frac{\rho(\slot,y)}{m_\rho(y)} .
    \end{align*}
    In fact, if we define the \emph{exponential map}~\cite[Corollary 1.8]{brown1964function}
    \begin{align*}
        &J \colon C(D\times E) \to C(E; C(D)) \\
        &J \colon f \mapsto (y \mapsto f(\slot,y)) ,
    \end{align*}
    then we have
    \begin{equation*}
        \Psi^\star(\rho,y) = J\bigl(\cG^\star(\rho)\bigr)(y) .
    \end{equation*}
\end{remark}
\begin{remark}[Weakened source condition]\label{remark:relaxed-Holder-cont}
	We note that Corollary~\ref{cor:holder_stability_incontext} remains true in a set $\widetilde S$ larger than $S$, namely
	\begin{equation*}
		\widetilde S \defeq \set{f\in C(D\times E)}{
			\sup_{x \in D} \| f(x, \slot) \|_{C^{0,\alpha}(E)} \leq R
		}.
	\end{equation*}
\end{remark}

\begin{example}[Loss of continuity for convergence in distribution]\label{example:conditioning-cannot-be-cont-in-weak-conv}
    It is harder for conditioning to be continuous in weaker topologies because the set of admissible measures becomes larger. 
    Write $\sP(\R^k)$ for the space of probability measures on $\R^k$ for some $k \in \N$.
    Consider the topology of convergence in distribution, denoted by $\weakly$, also known as the topology of \emph{weak} or \emph{narrow} convergence.
    Recall that it can be metrized by many distances, including Dudley's bounded-Lipschitz metric $\sfd_{\mathsf{BL}}$, or the Wasserstein metric $\sfd_{\mathsf{W}}$, for a compact state space.
	For every $n\in\N$, let
    \begin{align*}
        \nu_n\defeq \frac{1}{2}\delta_{(0,0)}+\frac{1}{2}\delta_{(1,1/n)}\qa\nu\defeq\frac{1}{2}\delta_{(0,0)}+\frac{1}{2}\delta_{(1,0)}
    \end{align*}
    be probability measures supported on $\R^2$. Then $\nu_n\weakly \nu$ as $n\to\infty$. But
    \begin{align*}
        \nu_n(\slot\condbar y=0) = \delta_0\qa \nu(\slot\condbar y=0)=\frac{1}{2}\delta_0+\frac{1}{2}\delta_1.
    \end{align*}
    So, the conditionals do not converge in distribution. Thus, $\Psi^\star(\slot,0)\colon (\sP(\R^2),\sfd_{\mathsf{BL}})\to (\sP(\R),\sfd_{\mathsf{BL}})$ is not a continuous map in the weak topology.
\end{example}

\section{Numerical experiment details}
\label{appendix:experimental}
This appendix describes details of the numerical experiments in \Cref{subsec:description-of-experiments}, data generation in \Cref{subsec:data-generation}, as well as architectural and training details in \Cref{subsec:architectural}. 
\subsection{Description of experiments}\label{subsec:description-of-experiments}
The two experiments involve training on $N=50000$ function-valued input-output pairs $\{(\rho_n, \kappa_n)\}_{n=1}^N$, given by Gaussian mixture joints $\rho_n$ together with the corresponding conditional kernel $\kappa_n$; note that both have closed-form expressions. Both are evaluated on a finite grid of $64^2$ points discretizing the domain $[-6,6]^2$. We use an additional $1000$ samples as validation during training, and an additional $1000$ as test data, which is not used in training. The data generation process is described in detail in \Cref{subsec:data-generation}. The \textit{correlated Gaussian} experimental setting corresponds to $K=1$, while the \textit{Gaussian mixture} setting to $K=3$, where $K$ represents the number of components in the mixture.

We also investigate the use of our trained models on sets of kernel density estimates (KDEs) obtained over $2000$ samples coming from test distributions in the Gaussian mixture class. 
This is done to compare our trained a models against a plug-in estimator for the conditional, the latter using the KDE.
We evaluate models trained on closed-form distributions on this set of KDEs, the latter viewed as functions. We describe details regarding the generation of KDE data in the following subsection.

\subsection{Data generation}\label{subsec:data-generation}
We describe the data generation process for the experimental settings in this paper. We begin by describing the generation mechanism for the closed-form density input-output pairs. For a given set of weights $\{w_k\}_{k=1}^K$ and a set of parameters $\{\mu_k^x, \mu_k^y, \sigma_k^x, \sigma_k^y, \xi_k\}_{k=1}^K$, consider a $K$-order Gaussian mixture joint distribution of the form 
\[
\rho(x,y)\defeq \sum_{k=1}^Kw_k\normal(x,y; \mu_k,\cC_k),\qw \mu_k \defeq \begin{pmatrix}
    \mu_k^x \\ \mu_k^y
\end{pmatrix}, \qa \cC_k \defeq \begin{pmatrix}
      \bigl(\sigma_k^{x}\bigr)^2 & \xi_k\sigma^x_k\sigma^y_k \\
      \xi_k\sigma^x_k\sigma^y_k & \bigl(\sigma_k^{y}\bigr)^2
    \end{pmatrix},
\]
for $\cN$ denoting the density function of a multivariate Gaussian. A standard computation leads to the closed form Markov kernel given by
\[
  \kappa_\rho(x \condbar y)
  = \sum_{k=1}^K
    \widetilde{w}_k(y)\,
    \mathcal{N}\left(
      x;\;
      \mu^x_k + \rho_k\,\frac{\sigma^x_k}{\sigma^y_k}(y-\mu^y_k),\;
      \bigl(\sigma^x_{k}\bigr)^2(1-\rho_k^2)
    \right),
\]
where
\[
\widetilde{w}_k(y) \defeq \frac{w_k\,\mathcal{N}\Bigl(y;\,\mu_k^y,\,\bigl(\sigma_k^{y}\bigr)^2\Bigr)}
         {\displaystyle\sum_{j=1}^K w_j\,\mathcal{N}\Bigl(y;\,\mu_j^y,\,\bigl(\sigma_j^{y}\bigr)^2\Bigr)}.
\]
For the correlated Gaussian ($K=1$) setting, the densities in the training, validation, and test sets are given by using $w_1=1$ and sampling independently $\mu_1^x$ and $\mu_1^y$ from $\mathrm{Uniform}(-3,3)$, $\sigma_1^x$ and $\sigma_1^x$ from $\mathrm{Uniform}(0.3,1.2)$, and $\xi_1$ from $\mathrm{Uniform}(-0.7,0.7)$. For the Gaussian mixture $(K=3)$ setting, the densities in the training, validation, and test sets are given by sampling the weights of the Gaussian mixtures independently from a $\mathrm{Dirichlet}(\mathbbm{1}_K)$ distribution, $\mu_k^x$ and $\mu_k^y$ from $\mathrm{Uniform}(-3,3)$, $\sigma_k^x$ and $\sigma_k^x$ from $\mathrm{Uniform}(0.3,1.2)$, and $\xi_k$ from $\mathrm{Uniform}(-0.7,0.7)$. For reproducibility, we note that all training, validation, test sets are generated using the seeds 0, 1, and 2, respectively.

In the KDE setting, for each testing instance, we randomly draw parameters defining a joint density $\rho=\rho(x,y)$ as before and then draw $2000$ samples from it, constructing a Gaussian KDE with Silverman's rule-of-thumb bandwidth. We then evaluate this KDE on the same grid as the training data. This serves as the model input, while the exact conditional density $\rho(x\condbar y) = \rho(x,y)/\int \rho(x,y)\,dx$, computed analytically, serves as the true kernel, against which we compare the model prediction. 
Last, the KDE plug-in estimator is computed by using the conditional formula $ \widehat{\rho}(x,y)/\int \widehat{\rho}(x,y)\,dx$ for some KDE $\widehat{\rho}$ and standard numerical integration.
For reproducibility, we note that these test sets are generated using the seed value 2.

\subsection{Architectural and training details}\label{subsec:architectural}
We report the architectural and training details for the FNO and TNO.

We use a 2D FNO from the ``NeuralOperator'' library \citep{kossaifi2025librarylearningneuraloperators} that takes as input the joint distribution discretized over a $64\times 64$ grid with two appended coordinate channels encoding the $x$- and $y$-grid values on the domain $[-6,6]$. The FNO we use here consists of $4$ Fourier layers with $16$ retained Fourier modes per spatial dimension and hidden channel width $128$, producing a single-channel output which we interpret as the predicted conditional density $\rho(x\condbar y)$. The model is trained with the Adam optimizer at an initial learning rate of $10^{-3}$, with a ReduceLROnPlateau scheduler (patience $5$, reduction factor $0.5$), minimizing the relative $L^1$ loss between the predicted and exact conditional densities. We use a batch size of $64$ for training.

We use a TNO from the GitHub repository of \cite{calvello2024continuum} that takes as input the joint distribution discretized over a $64\times 64$ grid with two appended coordinate channels encoding the $x$- and $y$-grid values on the domain $[-6,6]$, which serves as positional encoding. The TNOs we use consist of $12$ and $24$ self-attention layers for the correlated Gaussian and Gaussian mixture experiments, respectively. The layers have channel dimension $128$, $8$ attention heads, feedforward dimension $256$, and GELU activation. This produces an output sequence of shape $64\times 64$ which we interpret as $\rho(x\condbar y)$. Training follows the same optimizer, loss, and batch size as the FNO.

The FNO for the Gaussian mixture experiment was trained on a Tesla V100-PCIE-16GB on our institution's HPC cluster. All other architectures were trained with one NVIDIA H200-144GB GPU, on the same cluster. 


\begin{thebibliography}{63}
\providecommand{\natexlab}[1]{#1}
\providecommand{\url}[1]{\texttt{#1}}
\expandafter\ifx\csname urlstyle\endcsname\relax
  \providecommand{\doi}[1]{doi: #1}\else
  \providecommand{\doi}{doi: \begingroup \urlstyle{rm}\Url}\fi

\bibitem[Adams and Fournier(2003)]{adams2003sobolev}
R.~A. Adams and J.~J. Fournier.
\newblock \emph{Sobolev spaces}, volume 140.
\newblock Elsevier, 2 edition, 2003.

\bibitem[Al-Jarrah et~al.(2025)Al-Jarrah, Hosseini, and Taghvaei]{al2025fast}
M.~Al-Jarrah, B.~Hosseini, and A.~Taghvaei.
\newblock Fast filtering of non-{G}aussian models using amortized optimal transport maps.
\newblock \emph{IEEE Control Systems Letters}, 2025.

\bibitem[Bach et~al.(2025)Bach, Baptista, Calvello, Chen, and Stuart]{bach2025learning}
E.~Bach, R.~Baptista, E.~Calvello, B.~Chen, and A.~Stuart.
\newblock Learning enhanced ensemble filters.
\newblock \emph{Journal of Computational Physics}, art. 114550, 2025.

\bibitem[Bhattacharya et~al.(2021)Bhattacharya, Hosseini, Kovachki, and Stuart]{bhattacharya2021model}
K.~Bhattacharya, B.~Hosseini, N.~B. Kovachki, and A.~M. Stuart.
\newblock Model reduction and neural networks for parametric {PDEs}.
\newblock \emph{The SMAI Journal of Computational Mathematics}, 7:\penalty0 121--157, 2021.

\bibitem[Bhattacharya et~al.(2025)Bhattacharya, Cao, Stepaniants, Stuart, and Trautner]{bhattacharya2025learning}
K.~Bhattacharya, L.~Cao, G.~Stepaniants, A.~M. Stuart, and M.~Trautner.
\newblock Learning memory and material dependent constitutive laws.
\newblock \emph{preprint arXiv:2502.05463}, 2025.

\bibitem[Binder et~al.(2026)Binder, Dasgupta, and Oberai]{binder2026closed}
B.~Binder, A.~Dasgupta, and A.~Oberai.
\newblock Closed-form conditional diffusion models for data assimilation.
\newblock \emph{preprint arXiv:2603.21291}, 2026.

\bibitem[Biswal et~al.(2024)Biswal, Elamvazhuthi, and Sonthalia]{biswal2024universal}
S.~Biswal, K.~Elamvazhuthi, and R.~Sonthalia.
\newblock Universal approximation of mean-field models via transformers.
\newblock \emph{preprint arXiv:2410.16295}, 2024.

\bibitem[Bonev et~al.(2025)Bonev, Kurth, Mahesh, Bisson, Kossaifi, Kashinath, Anandkumar, Collins, Pritchard, and Keller]{bonev2025fourcastnet}
B.~Bonev, T.~Kurth, A.~Mahesh, M.~Bisson, J.~Kossaifi, K.~Kashinath, A.~Anandkumar, W.~D. Collins, M.~S. Pritchard, and A.~Keller.
\newblock {FourCastNet 3: A geometric approach to probabilistic machine-learning weather forecasting at scale}.
\newblock \emph{preprint arXiv:2507.12144}, 2025.

\bibitem[Boull{\'e} and Townsend(2024)]{boulle2024mathematical}
N.~Boull{\'e} and A.~Townsend.
\newblock A mathematical guide to operator learning.
\newblock In S.~Mishra and A.~Townsend, editors, \emph{Handbook of Numerical Analysis}, volume~25, pages 83--125. Elsevier, 2024.

\bibitem[Brezis(2011)]{brezis2011functional}
H.~Brezis.
\newblock \emph{Functional analysis, {Sobolev} spaces and partial differential equations}, volume~1.
\newblock Springer New York, NY, 2011.

\bibitem[Brown(1964)]{brown1964function}
R.~Brown.
\newblock Function spaces and product topologies.
\newblock \emph{The Quarterly Journal of Mathematics}, 15\penalty0 (1):\penalty0 238--250, 1964.

\bibitem[Brugiapaglia et~al.(2026)Brugiapaglia, Franco, and Nelsen]{brugiapaglia2026short}
S.~Brugiapaglia, N.~R. Franco, and N.~H. Nelsen.
\newblock A short tour of operator learning theory: {C}onvergence rates, statistical limits, and open questions.
\newblock \emph{preprint arXiv:2603.00819}, 2026.

\bibitem[Calvello et~al.(2025{\natexlab{a}})Calvello, Kovachki, Levine, and Stuart]{calvello2024continuum}
E.~Calvello, N.~B. Kovachki, M.~E. Levine, and A.~M. Stuart.
\newblock Continuum attention for neural operators.
\newblock \emph{Journal of Machine Learning Research}, 26\penalty0 (300):\penalty0 1--52, 2025{\natexlab{a}}.

\bibitem[Calvello et~al.(2025{\natexlab{b}})Calvello, Reich, and Stuart]{calvello2025ensemble}
E.~Calvello, S.~Reich, and A.~M. Stuart.
\newblock {Ensemble Kalman methods: A mean-field perspective}.
\newblock \emph{Acta Numerica}, 34:\penalty0 123--291, 2025{\natexlab{b}}.

\bibitem[Calvello et~al.(2026)Calvello, Carlson, {Kovachki}, Manta, and Stuart]{calvello2026operator}
E.~Calvello, E.~Carlson, N.~{Kovachki}, M.~N. Manta, and A.~M. Stuart.
\newblock Operator learning for smoothing and forecasting.
\newblock \emph{preprint arXiv:2603.20359}, 2026.

\bibitem[Castin et~al.(2024)Castin, Ablin, and Peyr{\'e}]{castin2024smooth}
V.~Castin, P.~Ablin, and G.~Peyr{\'e}.
\newblock How smooth is attention?
\newblock In \emph{International Conference on Machine Learning}, pages 5817--5840. PMLR, 2024.

\bibitem[Chen and Chen(1995)]{chen1995universal}
T.~Chen and H.~Chen.
\newblock Universal approximation to nonlinear operators by neural networks with arbitrary activation functions and its application to dynamical systems.
\newblock \emph{IEEE Transactions on Neural Networks}, 6\penalty0 (4):\penalty0 911--917, 1995.

\bibitem[Cole et~al.(2026)Cole, Wang, Chen, Lu, and Lai]{cole2026context}
F.~Cole, D.~Wang, Y.~Chen, Y.~Lu, and R.~Lai.
\newblock In-context operator learning on the space of probability measures.
\newblock \emph{preprint arXiv:2601.09979}, 2026.

\bibitem[Cui et~al.(2026)Cui, Feng, Pei, Wan, and Zhou]{cui2026amortized}
T.~Cui, X.~Feng, C.~Pei, X.~Wan, and T.~Zhou.
\newblock Amortized filtering and smoothing with conditional normalizing flows.
\newblock \emph{preprint arXiv:2604.07169}, 2026.

\bibitem[{de~Hoop} et~al.(2025){de~Hoop}, Kovachki, Lassas, and Nelsen]{de2025extension}
M.~V. {de~Hoop}, N.~B. Kovachki, M.~Lassas, and N.~H. Nelsen.
\newblock Extension and neural operator approximation of the electrical impedance tomography inverse map.
\newblock \emph{preprint arXiv:2511.20361}, 2025.

\bibitem[Dinh et~al.(2017)Dinh, Sohl-Dickstein, and Bengio]{dinh2017density}
L.~Dinh, J.~Sohl-Dickstein, and S.~Bengio.
\newblock Density estimation using real {NVP}.
\newblock In \emph{International Conference on Learning Representations}, 2017.

\bibitem[Dugundji(1951)]{dugundji1951extension}
J.~Dugundji.
\newblock An extension of {Tietze’s} theorem.
\newblock \emph{Pacific Journal of Mathematics}, 1\penalty0 (3):\penalty0 353--367, 1951.

\bibitem[Fraiman(2026)]{fraiman2026expressive}
D.~Fraiman.
\newblock On the expressive power of contextual relations in transformers.
\newblock \emph{preprint arXiv:2603.25860}, 2026.

\bibitem[Furuya et~al.(2025{\natexlab{a}})Furuya, {de Hoop}, and Lassas]{furuya2025transformers2}
T.~Furuya, M.~V. {de Hoop}, and M.~Lassas.
\newblock Transformers through the lens of support-preserving maps between measures.
\newblock \emph{preprint arXiv:2509.25611}, 2025{\natexlab{a}}.

\bibitem[Furuya et~al.(2025{\natexlab{b}})Furuya, {de Hoop}, and Peyr{\'e}]{furuya2025transformers1}
T.~Furuya, M.~V. {de Hoop}, and G.~Peyr{\'e}.
\newblock Transformers are universal in-context learners.
\newblock In \emph{International Conference on Learning Representations}, 2025{\natexlab{b}}.

\bibitem[Geshkovski et~al.(2025)Geshkovski, Letrouit, Polyanskiy, and Rigollet]{geshkovski2025mathematical}
B.~Geshkovski, C.~Letrouit, Y.~Polyanskiy, and P.~Rigollet.
\newblock A mathematical perspective on transformers.
\newblock \emph{Bulletin of the American Mathematical Society}, 62\penalty0 (3):\penalty0 427--479, 2025.

\bibitem[Gilbarg and Trudinger(2001)]{gilbarg2001elliptic}
D.~Gilbarg and N.~S. Trudinger.
\newblock \emph{Elliptic Partial Differential Equations of Second Order}.
\newblock Classics in Mathematics. Springer, 2001.

\bibitem[Gloeckler et~al.(2024)Gloeckler, Deistler, Weilbach, Wood, and Macke]{gloeckler2024all}
M.~Gloeckler, M.~Deistler, C.~D. Weilbach, F.~Wood, and J.~H. Macke.
\newblock All-in-one simulation-based inference.
\newblock In \emph{International Conference on Machine Learning}, pages 15735--15766. PMLR, 2024.

\bibitem[Grafakos(2008)]{grafakos2008classical}
L.~Grafakos.
\newblock \emph{Classical {F}ourier analysis}, volume~2.
\newblock Springer, 2008.

\bibitem[Huan et~al.(2024)Huan, Jagalur, and Marzouk]{huan2024optimal}
X.~Huan, J.~Jagalur, and Y.~Marzouk.
\newblock Optimal experimental design: {F}ormulations and computations.
\newblock \emph{Acta Numerica}, 33:\penalty0 715--840, 2024.

\bibitem[Huang et~al.(2025)Huang, Nelsen, and Trautner]{huang2025operator}
D.~Z. Huang, N.~H. Nelsen, and M.~Trautner.
\newblock An operator learning perspective on parameter-to-observable maps.
\newblock \emph{Foundations of Data Science}, 7\penalty0 (1):\penalty0 163--225, 2025.

\bibitem[Ilin and Sushko(2026)]{ilin2026discoformer}
V.~Ilin and P.~Sushko.
\newblock {DiScoFormer: P}lug-in density and score estimation with transformers.
\newblock In \emph{International Conference on Machine Learning}, 2026.

\bibitem[Kawata and Suzuki(2026)]{kawata2026transformers}
R.~Kawata and T.~Suzuki.
\newblock Transformers as measure-theoretic associative memory: {A} statistical perspective and minimax optimality.
\newblock \emph{preprint arXiv:2602.01863}, 2026.

\bibitem[Kossaifi et~al.(2025)Kossaifi, Kovachki, Li, Pitt, Liu-Schiaffini, George, Bonev, Azizzadenesheli, Berner, Duruisseaux, and Anandkumar]{kossaifi2025librarylearningneuraloperators}
J.~Kossaifi, N.~Kovachki, Z.~Li, D.~Pitt, M.~Liu-Schiaffini, R.~J. George, B.~Bonev, K.~Azizzadenesheli, J.~Berner, V.~Duruisseaux, and A.~Anandkumar.
\newblock A library for learning neural operators.
\newblock \emph{preprint arXiv:2412.10354}, 2025.

\bibitem[Kovachki et~al.(2023)Kovachki, Li, Liu, Azizzadenesheli, Bhattacharya, Stuart, and Anandkumar]{kovachki2021neural}
N.~B. Kovachki, Z.~Li, B.~Liu, K.~Azizzadenesheli, K.~Bhattacharya, A.~M. Stuart, and A.~Anandkumar.
\newblock {Neural operator: Learning maps between function spaces with applications to PDEs}.
\newblock \emph{Journal of Machine Learning Research}, 24\penalty0 (89):\penalty0 1--97, 2023.

\bibitem[Kovachki et~al.(2024)Kovachki, Lanthaler, and Stuart]{kovachki2024operator}
N.~B. Kovachki, S.~Lanthaler, and A.~M. Stuart.
\newblock Operator learning: {A}lgorithms and analysis.
\newblock In S.~Mishra and A.~Townsend, editors, \emph{Handbook of Numerical Analysis}, volume~25, pages 419--467. Elsevier, 2024.

\bibitem[Kurth et~al.(2023)Kurth, Subramanian, Harrington, Pathak, Mardani, Hall, Miele, Kashinath, and Anandkumar]{kurth2023fourcastnet}
T.~Kurth, S.~Subramanian, P.~Harrington, J.~Pathak, M.~Mardani, D.~Hall, A.~Miele, K.~Kashinath, and A.~Anandkumar.
\newblock {FourCastNet: Accelerating global high-resolution weather forecasting using adaptive Fourier neural operators}.
\newblock In \emph{Proceedings of the Platform for Advanced Scientific Computing Conference}, pages 1--11, 2023.

\bibitem[Lanthaler et~al.(2022)Lanthaler, Mishra, and Karniadakis]{lanthaler2022error}
S.~Lanthaler, S.~Mishra, and G.~E. Karniadakis.
\newblock Error estimates for {DeepONets: A} deep learning framework in infinite dimensions.
\newblock \emph{Transactions of Mathematics and its Applications}, 6\penalty0 (1):\penalty0 tnac001, 2022.

\bibitem[Lanthaler et~al.(2025)Lanthaler, Li, and Stuart]{lanthaler2025nonlocality}
S.~Lanthaler, Z.~Li, and A.~M. Stuart.
\newblock Nonlocality and nonlinearity implies universality in operator learning.
\newblock \emph{Constructive Approximation}, 62\penalty0 (2):\penalty0 261--303, 2025.

\bibitem[Le~Gall(2016)]{le2016brownian}
J.-F. Le~Gall.
\newblock \emph{Brownian motion, martingales, and stochastic calculus}, volume 274.
\newblock Springer, 2016.

\bibitem[Li et~al.(2021)Li, Kovachki, Azizzadenesheli, Liu, Bhattacharya, Stuart, and Anandkumar]{li2021fourier}
Z.~Li, N.~B. Kovachki, K.~Azizzadenesheli, B.~Liu, K.~Bhattacharya, A.~M. Stuart, and A.~Anandkumar.
\newblock Fourier neural operator for parametric partial differential equations.
\newblock In \emph{International Conference on Learning Representations}, 2021.

\bibitem[Li et~al.(2022)Li, Meunier, Mollenhauer, and Gretton]{li2022optimal}
Z.~Li, D.~Meunier, M.~Mollenhauer, and A.~Gretton.
\newblock Optimal rates for regularized conditional mean embedding learning.
\newblock In \emph{Advances in Neural Information Processing Systems}, volume~35, pages 4433--4445, 2022.

\bibitem[Li et~al.(2023)Li, Kovachki, Choy, Li, Kossaifi, Otta, Nabian, Stadler, Hundt, Azizzadenesheli, and Anandkumar]{li2023geometry}
Z.~Li, N.~Kovachki, C.~Choy, B.~Li, J.~Kossaifi, S.~Otta, M.~A. Nabian, M.~Stadler, C.~Hundt, K.~Azizzadenesheli, and A.~Anandkumar.
\newblock Geometry-informed neural operator for large-scale {3D PDEs}.
\newblock In \emph{Advances in Neural Information Processing Systems}, volume~36, pages 35836--35854, 2023.

\bibitem[Lipman et~al.(2023)Lipman, Chen, Ben-Hamu, Nickel, and Le]{lipman2022flow}
Y.~Lipman, R.~T.~Q. Chen, H.~Ben-Hamu, M.~Nickel, and M.~Le.
\newblock Flow matching for generative modeling.
\newblock In \emph{International Conference on Learning Representations}, 2023.

\bibitem[Lu et~al.(2021)Lu, Jin, Pang, Zhang, and Karniadakis]{lu2021learning}
L.~Lu, P.~Jin, G.~Pang, Z.~Zhang, and G.~Karniadakis.
\newblock Learning nonlinear operators via {DeepONet} based on the universal approximation theorem of operators.
\newblock \emph{Nature Machine Intelligence}, 3:\penalty0 218--229, 2021.

\bibitem[Lyu et~al.(2026)Lyu, Yu, and Schaeffer]{lyu2026mvnn}
L.~Lyu, X.~Yu, and H.~Schaeffer.
\newblock {MVNN: A} measure-valued neural network for learning {McKean-Vlasov} dynamics from particle data.
\newblock \emph{preprint arXiv:2604.00333}, 2026.

\bibitem[{Marzouk} et~al.(2017){Marzouk}, Moselhy, Parno, and Spantini]{marzouk2017sampling}
Y.~{Marzouk}, T.~Moselhy, M.~Parno, and A.~Spantini.
\newblock Sampling via measure transport: {A}n introduction.
\newblock \emph{Springer Books}, pages 785--825, 2017.

\bibitem[Mirza and Osindero(2014)]{mirza2014conditional}
M.~Mirza and S.~Osindero.
\newblock Conditional generative adversarial nets.
\newblock \emph{preprint arXiv:1411.1784}, 2014.

\bibitem[Moosm{\"u}ller and Cloninger(2023)]{moosmuller2023linear}
C.~Moosm{\"u}ller and A.~Cloninger.
\newblock Linear optimal transport embedding: {Provable W}asserstein classification for certain rigid transformations and perturbations.
\newblock \emph{Information and Inference: A Journal of the IMA}, 12\penalty0 (1):\penalty0 363--389, 2023.

\bibitem[Nelsen and Yang(2026)]{nelsen2026hna}
N.~H. Nelsen and Y.~Yang.
\newblock {Operator learning meets inverse problems: A probabilistic perspective}.
\newblock In A.~Hauptmann, B.~Jin, and C.-B. Sch{\"o}nlieb, editors, \emph{Handbook of Numerical Analysis}, volume~27. Elsevier, 2026.

\bibitem[Panaretos and Zemel(2020)]{panaretos2020invitation}
V.~M. Panaretos and Y.~Zemel.
\newblock \emph{An invitation to statistics in {W}asserstein space}.
\newblock Springer Nature, 2020.

\bibitem[Pierret et~al.(2026)Pierret, Tosel, Delon, and Newson]{pierret2026flow}
E.~Pierret, V.~Tosel, J.~Delon, and A.~Newson.
\newblock Flow matching for applied mathematicians, 2026.

\bibitem[Revach et~al.(2022)Revach, Shlezinger, Ni, Escoriza, Van~Sloun, and Eldar]{revach2022kalmannet}
G.~Revach, N.~Shlezinger, X.~Ni, A.~L. Escoriza, R.~J. Van~Sloun, and Y.~C. Eldar.
\newblock {KalmanNet: Neural network aided Kalman filtering for partially known dynamics}.
\newblock \emph{IEEE Transactions on Signal Processing}, 70:\penalty0 1532--1547, 2022.

\bibitem[Sanz-Alonso et~al.(2023)Sanz-Alonso, Stuart, and Taeb]{sanz2023inverse}
D.~Sanz-Alonso, A.~Stuart, and A.~Taeb.
\newblock \emph{Inverse problems and data assimilation}, volume 107.
\newblock Cambridge University Press, 2023.

\bibitem[Schuster et~al.(2020)Schuster, Mollenhauer, Klus, and Muandet]{schuster2020kernel}
I.~Schuster, M.~Mollenhauer, S.~Klus, and K.~Muandet.
\newblock Kernel conditional density operators.
\newblock In \emph{International Conference on Artificial Intelligence and Statistics}, pages 993--1004. PMLR, 2020.

\bibitem[Sohn et~al.(2015)Sohn, Yan, and Lee]{sohn2015learning}
K.~Sohn, X.~Yan, and H.~Lee.
\newblock Learning structured output representation using deep conditional generative models.
\newblock In \emph{Advances in Neural Information Processing Systems}, volume~28, 2015.

\bibitem[Song et~al.(2009)Song, Huang, Smola, and Fukumizu]{song2009hilbert}
L.~Song, J.~Huang, A.~Smola, and K.~Fukumizu.
\newblock Hilbert space embeddings of conditional distributions.
\newblock In \emph{International Conference on Machine Learning}, 2009.

\bibitem[Stuart(2010)]{stuart2010inverse}
A.~M. Stuart.
\newblock Inverse problems: {A B}ayesian perspective.
\newblock \emph{Acta Numerica}, 19:\penalty0 451--559, 2010.

\bibitem[Szab{\'o} et~al.(2016)Szab{\'o}, Sriperumbudur, P{\'o}czos, and Gretton]{szabo2016learning}
Z.~Szab{\'o}, B.~K. Sriperumbudur, B.~P{\'o}czos, and A.~Gretton.
\newblock Learning theory for distribution regression.
\newblock \emph{Journal of Machine Learning Research}, 17\penalty0 (152):\penalty0 1--40, 2016.

\bibitem[Tong et~al.(2026)Tong, Wang, and Yan]{tong2026latent}
X.~T. Tong, Y.~Wang, and L.~Yan.
\newblock Latent autoencoder ensemble {K}alman filter for data assimilation.
\newblock \emph{preprint arXiv:2603.06752}, 2026.

\bibitem[Whittle et~al.(2026)Whittle, Ziomek, Rawling, and Osborne]{whittle2026distribution}
G.~Whittle, J.~Ziomek, J.~Rawling, and M.~A. Osborne.
\newblock Distribution transformers: {F}ast approximate {B}ayesian inference with on-the-fly prior adaptation.
\newblock In \emph{International Conference on Machine Learning}, 2026.

\bibitem[Wu et~al.(2023)Wu, Chen, and {Ghattas}]{wu2023fast}
K.~Wu, P.~Chen, and O.~{Ghattas}.
\newblock A fast and scalable computational framework for large-scale high-dimensional {B}ayesian optimal experimental design.
\newblock \emph{SIAM/ASA Journal on Uncertainty Quantification}, 11\penalty0 (1):\penalty0 235--261, 2023.

\bibitem[Zeng et~al.(2025)Zeng, Zhang, Zhou, Wang, Wang, Liu, Wu, and Huang]{zeng2025point}
C.~Zeng, Y.~Zhang, J.~Zhou, Y.~Wang, Z.~Wang, Y.~Liu, L.~Wu, and D.~Z. Huang.
\newblock Point cloud neural operator for parametric {PDE}s on complex and variable geometries.
\newblock \emph{Computer Methods in Applied Mechanics and Engineering}, 443:\penalty0 118022, 2025.

\end{thebibliography}
\end{document}